\documentclass{article}

\PassOptionsToPackage{numbers, compress}{natbib}
\usepackage{natbib}

\usepackage{arxiv}

\usepackage[utf8]{inputenc} 
\usepackage[T1]{fontenc}    
\usepackage{hyperref}       
\usepackage{url}            
\usepackage{booktabs}       
\usepackage{amsfonts}       
\usepackage{nicefrac}       
\usepackage{microtype}      
\usepackage{xcolor} 
\usepackage{amsmath}
\usepackage{amsthm}
\usepackage[normalem]{ulem}

\usepackage{hyperref}
\usepackage{amsfonts}
\usepackage{graphicx}

\usepackage[multiple]{footmisc}
\usepackage{booktabs}
\usepackage{dsfont}
\usepackage[utf8]{inputenc}
\usepackage{balance}
\usepackage{siunitx}
\usepackage{multirow}
\usepackage{graphicx}
\usepackage{listings}
\usepackage{commath}
\usepackage{epstopdf}
\usepackage{color}
\usepackage{url}
\usepackage{caption}
\usepackage{alltt}
\usepackage{mathrsfs}
\usepackage{xspace}
\usepackage{float}
\usepackage{subcaption}
\usepackage{graphicx,fancyvrb}
\usepackage{algorithm}
\usepackage{algpseudocode}

\usepackage{amsmath}
\usepackage{amsfonts}
\usepackage{xcolor}
\usepackage{tikz}
\usetikzlibrary{shapes.geometric, backgrounds, fit, calc, intersections, hobby, ext.paths.ortho, arrows.meta, positioning}

\definecolor{abstractDraw}{RGB}{0, 90, 180}
\colorlet{possiblecol}{green!50!black}
\colorlet{possiblecollight}{green!50}
\colorlet{abstractcol}{orange!80!black}
\colorlet{abstractcollight}{orange!10}
\colorlet{operationcolor}{gray!30}
\colorlet{certcol}{blue!50!black}
\colorlet{certcollight}{blue!15}

\usepackage{amsthm}
\usepackage[acronym]{glossaries}
\usepackage{glossaries-extra}
\setkeys{glslink}{hyper=false}

\theoremstyle{definition}

\newtheorem{example}{Example}[section]

\theoremstyle{plain}
\newtheorem{proposition}{Proposition}[section]

\usepackage{booktabs}

\usepackage{mathtools}
\usepackage{caption}
\usepackage{float}
\usepackage{capt-of}
\usepackage{booktabs}
\usepackage{colortbl}
\usepackage{xcolor}
\usepackage{xfrac}
\usepackage{footnote}

 \usepackage{enumitem}
\usepackage{array}
\usepackage{arydshln}

\usepackage{cleveref}

\usepackage{todonotes}

\definecolor{mygreen}{rgb}{0,0.6,0}
\definecolor{myred}{rgb}{0.6,0,0}
\definecolor{mygray}{rgb}{0.5,0.5,0.5}
\definecolor{mymauve}{rgb}{0.58,0,0.82}
\definecolor{myblue}{rgb}{0,0,1}

\usepackage{listings}
\lstnewenvironment{VerbatimText}[1][]{
    
    \lstset{fancyvrb=true,basicstyle=\footnotesize,captionpos=b,xleftmargin=2em,#1}
}{}

\lstdefinestyle{skillprompt}{
    basicstyle=\ttfamily\footnotesize,
    breaklines=true,
    columns=fullflexible,
    frame=single,
    framerule=0.3pt,
    rulecolor=\color{black!30},
    backgroundcolor=\color{black!3},
    xleftmargin=0.5em,
    xrightmargin=0.5em,
    aboveskip=0.7em,
    belowskip=0.7em
}

\usepackage{booktabs}
\usepackage{pgfplots}
\usepackage{pgfplotstable}

\PassOptionsToPackage{hyphens}{url}\usepackage{hyperref}

\usepackage{amsmath}

\newcommand{\sys}{\mbox{\textsc{DivSkill}-SQL}\xspace}
\newcommand{\chasesql}{\mbox{\textsc{CHASE}-SQL}\xspace}
\newcommand{\reforce}{\mbox{\textsc{ReFoRCE}}\xspace}

\newcommand{\dinsql}{\mbox{\textsc{DIN}-SQL}\xspace}

\newcommand{\ignore}[1]{}

\definecolor{black}{rgb}{0,0,0}
\definecolor{grey}{rgb}{0.8,0.8,0.8}
\definecolor{red}{rgb}{1,0,0}
\definecolor{green}{rgb}{0,1,0}
\definecolor{darkgreen}{rgb}{0,0.5,0}
\definecolor{darkpurple}{rgb}{0.5,0,0.5}
\definecolor{darkdarkpurple}{rgb}{0.3,0,0.3}
\definecolor{blue}{rgb}{0,0,1}
\definecolor{shadegreen}{rgb}{0.95,1,0.95}
\definecolor{shadeblue}{rgb}{0.95,0.95,1}
\definecolor{shadered}{rgb}{1,0.85,0.85}
\definecolor{shadegrey}{rgb}{0.85,0.85,0.85}
\definecolor{oddRowGrey}{rgb}{0.80,0.80,0.80}
\definecolor{evenRowGrey}{rgb}{0.85,0.85,0.85}
\definecolor{lightpurple}{rgb}{0.88,1.0,1.0}

\definecolor{syscolor}{rgb}{0.10, 0.37, 0.49}
\definecolor{epstwocolor}{rgb}{0.80, 0.40, 0.10}
\definecolor{epsthreecolor}{rgb}{0.40, 0.15, 0.55}
\definecolor{uncertaincolor}{rgb}{0.5, 0.1, 0.1}
\definecolor{cleancolor}{rgb}{0.20, 0.55, 0.25}

\newcommand{\RNum}[1]{\uppercase\expandafter{\romannumeral #1\relax}}

\newtheorem{col}{Colollary}

\newcommand{\proj}[1]{{\Pi}}
\newcommand{\sel}[1]{{\sigma}}

\newcommand{\cut}[1]{}
\newcommand{\eat}[1]{}

\usepackage{mdframed}

\usepackage[multiple]{footmisc}

\usepackage[most]{tcolorbox}
\usepackage{xcolor}

\newtcolorbox{takeaway}[1][]{
  enhanced,
  breakable,
  colback=black!10,
  colframe=black,
  boxrule=0.7pt,
  left=1.5mm, right=1.5mm, top=0mm, bottom=0mm,
  title={\textcolor{black}{\textbf{Key takeaway:}\,\if\relax\detokenize{#1}\relax\else:~#1\fi}},
  fonttitle=\normalsize,
  attach title to upper,
  before skip=4pt,
  after skip=4pt,
}

\usepackage[most]{tcolorbox}

\definecolor{sqlkw}{RGB}{30,90,170}

\newtcolorbox{seedskill}[1][]{
  enhanced jigsaw,
  breakable,
  sharp corners,
  boxrule=0.4pt,
  colback=gray!4,
  colframe=gray!50,
  colbacktitle=gray!18,
  coltitle=black,
  fonttitle=\bfseries\ttfamily\footnotesize,
  fontupper=\ttfamily\footnotesize,
  title={Seed prompt: \texttt{\detokenize{#1}}},
  attach boxed title to top left={xshift=4pt, yshift=-2pt},
  boxed title style={
    colback=gray!18,
    colframe=gray!18,
    sharp corners,
    boxrule=0pt
  },
  left=5pt,
  right=5pt,
  top=3pt,
  bottom=3pt,
  before skip=4pt,
  after skip=4pt
}

\newtcolorbox{optskill}[1][]{
  enhanced jigsaw,
  breakable,
  sharp corners,
  boxrule=0.4pt,
  colback=blue!3,
  colframe=blue!35,
  colbacktitle=blue!12,
  coltitle=black,
  fonttitle=\bfseries\ttfamily\footnotesize,
  fontupper=\ttfamily\footnotesize,
  title={Optimized prompt: \texttt{\detokenize{#1}}},
  attach boxed title to top left={xshift=4pt, yshift=-2pt},
  boxed title style={
    colback=blue!12,
    colframe=blue!12,
    sharp corners,
    boxrule=0pt
  },
  left=5pt,
  right=5pt,
  top=3pt,
  bottom=3pt,
  before skip=4pt,
  after skip=4pt
}

\usepackage{tikz}
\usetikzlibrary{arrows.meta,positioning,shapes.geometric,fit,backgrounds,decorations.pathreplacing}
\usepackage{wrapfig}

\newacronym{CQA}{CQA}{consistent query answering}
\newacronym{ML}{ML}{machine learning}
\newacronym{cpclean}{CPClean}{CPClean}
\newacronym{MSE}{MSE}{mean squared error}
  
\glsunset{cpclean}
 
  {
  \noindent \textsc{Proof Sketch.}%
  }%
  {\qedsymbol}

\definecolor{lightblue}{RGB}{220,225,255}

\title{Residual Skill Optimization \\ for Text-to-SQL Ensembles}
\renewcommand{\shorttitle}{Residual Skill Optimization for Text-to-SQL Ensembles}

\begin{document}



\newcommand{\affUCSD}{\textsuperscript{\normalfont 1}}
\newcommand{\affSnow}{\textsuperscript{\normalfont 2}}
\newcommand{\markEqual}{\textsuperscript{*}}
\newcommand{\markSnowWork}{\textsuperscript{\textdagger}}
\newcommand{\markSenior}{\textsuperscript{\textdaggerdbl}}

\author{
\begin{tabular}{c}
Jiongli Zhu\affUCSD\markEqual\markSnowWork \quad
Haoquan Guan\affUCSD\markEqual\markSnowWork \quad
Parjanya Prajakta Prashant\affUCSD\markEqual\markSnowWork \quad
Nikki Lijing Kuang\affSnow \\
Seyedeh Baharan Khatami\affUCSD\markSnowWork \quad
Canwen Xu\affSnow \quad
Xiaodong Yu\affSnow \quad
Yingyu Lin\affUCSD \\
Zhewei Yao\affSnow \quad
Yuxiong He\affSnow\markSenior \quad
Babak Salimi\affUCSD\markSnowWork\markSenior \\
\normalfont\textsuperscript{1}University of California, San Diego \quad
\normalfont\textsuperscript{2}Snowflake AI Research
\end{tabular}
}

\maketitle

\begingroup
\renewcommand{\thefootnote}{\fnsymbol{footnote}}
\footnotetext[1]{Equal contribution.}
\footnotetext[2]{Work done while working at Snowflake AI Research.}
\footnotetext[3]{Co-senior authors.}
\endgroup
\setcounter{footnote}{0}


\begin{abstract}

Text-to-SQL ensembles improve over single-candidate generation by drawing multiple SQL candidates and selecting one, but their effectiveness is bounded by Pass@K, the probability that at least one of $K$ candidates is correct. Existing methods source diversity heuristically through stochastic decoding or prompt variants, leaving candidate sets dominated by correlated failures. We present \sys, a residual skill optimization framework that builds complementary agentic Text-to-SQL ensembles without model fine-tuning: each new skill is optimized on examples the current skill ensemble fails on, provably targeting its marginal contribution to Pass@K. On Spider2-Lite, \sys improves selected accuracy by up to $+11.1$ points on Snowflake and $+8.3$ on BigQuery over the strongest ensemble baseline, with consistent gains across two base models (Opus-4.6 and GPT-5.4). Skills optimized on a single dialect transfer without retraining across dialects (Snowflake, BigQuery, SQLite) and to a different task formulation, such as BIRD-Critic ($+2.6$ pts). 
Error diagnostics show up to $3\times$ fewer hallucinated schema references and function calls, indicating that gains come from genuinely reliable complementary skills rather than surface-form variation.
\end{abstract}

\section{Introduction}\label{sec:intro}

Text-to-SQL translates natural language questions into executable SQL queries, making relational databases accessible without SQL expertise. While large language models (LLMs) achieve strong results on standard benchmarks~\citep{yu2018spider,pourreza2023evaluating}, real-world queries, which involve large schemas, dialect-specific syntax, and multi-step query logic, remain difficult. 
This has motivated agentic Text-to-SQL systems that inspect schemas, execute intermediate queries, observe feedback, and iteratively repair errors~\citep{talaei2024chess,deng2025reforce}.

Yet even with multi-step interaction, no single agentic execution reliably solves every query, which has motivated ensembling as a complementary axis of robustness~\citep{pourreza2024chase,yang2025mars,liu2026xiyan,deng2025reforce}: rather than committing to a single SQL generation, ensembles generate multiple SQL candidates and select a final answer. The potential of an ensemble is bounded above by Pass@K, the probability that at least one of $K$ candidates is correct~\citep{chen2021evaluating}: no selector, however good, can recover an answer that was never generated. Pass@K therefore captures what the generation stage is responsible for, separately from selection.

Existing Text-to-SQL ensembles produce candidate diversity by combining stochastic decoding, hand-designed prompt or workflow variants, and in some cases multiple fine-tuned generators~\citep{pourreza2024chase,liu2026xiyan}. Whether the resulting candidates recover different failure cases, however, is left to chance: nothing in the procedure pushes new candidates to solve examples that earlier ones miss. 
In agentic systems this is especially fragile: randomness in early planning propagates through long reasoning trajectories, so high-temperature sampling produces noisier variants of the same path rather than complementary solutions, and introduces unstable reasoning, spurious joins, and dialect errors that degrade the candidate pool~\citep{pourreza2025reasoning}. Pass@K accordingly plateaus, with little gain from additional candidates.

To address this, we introduce \sys, a residual skill optimization framework that makes candidate complementarity an explicit optimization target instead of leaving it to heuristics or chance. A \emph{skill} is a high-level instruction file that controls the agent's decomposition style, schema-exploration policy, drafting strategy, and repair logic; \sys operates entirely at this prompt level, with no model fine-tuning. Starting from a base skill, \sys evaluates the agent on the training set, identifies the unresolved examples, and uses reflective prompt optimization~\citep{agrawal2025gepa} to refine a new skill specifically on this residual; it repeats round by round until a pool of $K$ complementary skills has been learned. 
A new skill need not to be globally better than its predecessors; it is useful precisely when it recovers examples they miss. At inference, \sys runs each learned skill and selects a final SQL from the resulting candidate set via pairwise comparison~\citep{pourreza2024chase}. Each skill is trained to cover what the others miss, so the ensemble is complementary by construction rather than by chance.

\begin{figure*}
    \centering
    \includegraphics[width=\textwidth]{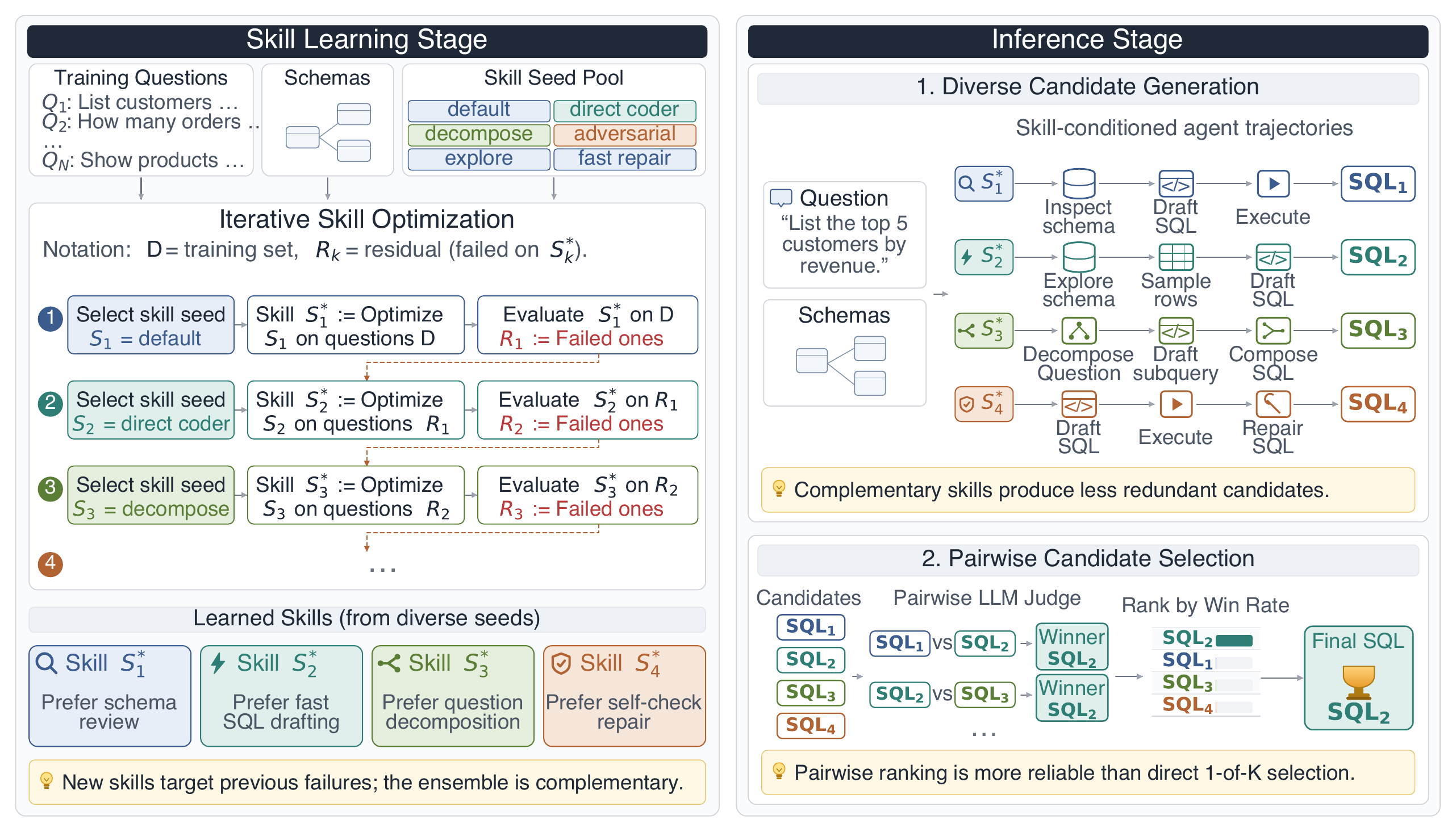}
    \vspace{-3mm}
    \caption{
    System diagram of \sys. 
    The left panel shows skill construction: starting from diverse strategy prompts, the system repeatedly identifies unsolved questions and refines the next skill toward those remaining cases. 
    The right panel shows test-time execution: multiple skill-guided agents solve the same Text-to-SQL problem through different interaction patterns, producing SQL candidates that are then compared to choose the final output. 
    }
    \label{fig:diagram}
    \vspace{-3mm}
\end{figure*}

We evaluate \sys on recent Text-to-SQL benchmarks across four dialects: Spider2-Lite over SQLite, Snowflake, and BigQuery, and BIRD-Critic over PostgreSQL. \sys improves selected accuracy on Spider2-Lite by up to \textbf{+11.1} points on Snowflake and \textbf{+8.3} on BigQuery over the strongest ensemble baseline, with consistent gains across two base models (Opus-4.6 and GPT-5.4). 
Skills optimized on standard Text-to-SQL in a single dialect (Snowflake) transfer without retraining to BigQuery and SQLite. Moreover, skills optimized for standard Text-to-SQL generalize to the debugging-style generation setting of BIRD-Critic, where \sys improves accuracy by \textbf{+2.6}.
Error diagnostics show \textbf{up to 3$\times$} fewer hallucinated schema references and unsupported-function calls, and tool-use trajectory analysis shows that \sys reduces redundancy in agent behavior by $19\%$--$28\%$, producing more diverse schema-inspection, decomposition, drafting, execution, and repair patterns than repeated runs of the same agent.

Our contributions are as follows:
\begin{itemize}
    \item We formulate Text-to-SQL ensembling as residual Pass@K optimization over agent skills, where the goal is to learn complementary behaviors that cover different failure modes.
    \item We propose \sys, a residual skill optimization framework that improves candidate-set coverage without model fine-tuning or high-temperature sampling.
    \item We show that optimizing each new skill on unresolved examples directly targets its marginal contribution to Pass@K.
    \item We evaluate \sys across recent Text-to-SQL benchmarks and four SQL dialects, showing improved accuracy and reduced redundancy in agent trajectories.
\end{itemize}

\section{Related Works}\label{sec:related}
\vspace{-2mm}

\paragraph{Pass@K Optimization}
Pass@K measures whether a model produces at least one correct solution among $K$ samples~\citep{chen2021evaluating}, and has recently been adopted as an optimization target. \citet{yue2025does} show that RL with verifiable rewards improves Pass@1 but not Pass@K: training reduces output diversity and leaves hard problems outside the base model's support unsolved. This motivates directly optimizing for diversity, which several works pursue through policy optimization~\citep{chen2025pass,walder2025pass,yao2025diversity}. However, these methods all require parameter updates, limiting them to open-weight models and to improving a single model's output distribution. 
Our method instead learns $K$ complementary skills without parameter updates, improving Pass@K through ensemble construction for both open and closed models.

\vspace{-2mm}
\paragraph{Skill and Prompt Optimization}
Prompts strongly influence LLM behavior~\citep{brown2020language}, motivating automatic optimization methods that improve prompts from task feedback rather than manual design. Early work optimizes discrete trigger tokens or textual prompts~\citep{shin2020autoprompt}, while recent methods use LLMs to propose, evaluate, and revise prompts iteratively~\citep{zhou2022large,yang2023large,agrawal2025gepa}. Boosted prompt ensembles~\citep{pitis2023boosted} also optimize prompt ensembles by adding few-shot prompts that target examples on which the current ensemble is uncertain or incorrect. Beyond short prompts and few-shot demonstrations, agent systems increasingly use \emph{skills}: modular instruction files encoding task strategies, constraints, and tool-use policies~\citep{zhang2026equipping}. Recent work explores agents that rewrite their skills~\citep{zhou2026memento}, jointly optimize skills with model parameters~\citep{xia2026skillrl}, or evolve and reuse skills over time~\citep{ma2026skillclaw,zhang2026evoskills,zhang2026memskill}. 
Our work follows this direction but targets ensemble construction: rather than learning one strong skill, we learn complementary skills that address different failure modes and improve the ensemble's Pass@K.

\vspace{-2mm}
\paragraph{Text-to-SQL}
Early LLM-based Text-to-SQL methods prompt a single model to generate SQL directly~\citep{dong2023c3,tai2023exploring,10.14778/3681954.3682003,lisigmodcodes,chensigmodreliable}. More recent systems use agent-based pipelines that decompose the task into schema pruning, evidence extraction, SQL generation, execution, and refinement~\citep{pourreza2023din,gao2023text,wang2025mac,talaei2024chess,xie-etal-2024-decomposition,dai-etal-2025-parsql,zhang2025structureguided}. Several works further improve performance through ensembling~\citep{pourreza2024chase,deng2025reforce,yang2025mars,liu2026xiyan,lee-etal-2025-mcs,guo-etal-2025-sqlforge}: CHASE-SQL generates candidates from diverse prompts and selects among them via a tournament~\citep{pourreza2024chase}; MARS-SQL trains a multi-agent system with RL~\citep{yang2025mars,ma2026sqlr,zhai-etal-2025-optimizing}; and XiYan-SQL fine-tunes multiple generators to induce diversity~\citep{liu2026xiyan,yang-etal-2024-synthesizing}. In contrast, our method requires no weight modifications or hand-designed ensembles. 
We optimize skills to explicitly construct complementary generators, increasing Pass@K coverage while remaining applicable to both open and closed LLMs.

\vspace{-3mm}
\section{Method}\label{sec:method}
\vspace{-3mm}
\subsection{Setup and Notation}\label{sec:method:setup}
\vspace{-2mm}
\paragraph{Text-to-SQL with agentic execution.}
A Text-to-SQL instance is a pair $(q, \mathcal{D})$ consisting of a natural-language question $q$ and a relational database $\mathcal{D}$, and the goal is to produce an executable SQL query whose result on $\mathcal{D}$ matches that of a gold reference. We follow the agentic execution paradigm of recent Text-to-SQL systems~\citep{talaei2024chess,deng2025reforce}: rather than generating SQL in a single forward pass, an \emph{agent} interleaves tool calls, such as inspecting schema, sampling rows, drafting candidate SQL, and repairing errors, over multiple steps before returning a final query. 

\vspace{-2mm}
\paragraph{Skills.}
We modulate the agent's behavior through \emph{skills}. A skill $s$ is a high-level instruction file (a system prompt expressed in natural language) that controls the agent's reasoning and tool-use policy: which decomposition style to favor, when to explore the schema versus draft directly, what repair patterns to apply on execution errors, and so on. We write $a_s$ for the agent equipped with skill $s$ and treat $s$ as identified with its prompt $\pi_s$, so that the optimization space $\mathcal{S}$ is the space of natural-language instruction files. Two distinct skills induce two genuinely different agent trajectories on the same input, not merely two stochastic samples of the same trajectory. 

\begin{example}[Skill Examples]\label{ex:seed-skills}
We show two simplified skills below that differ not only in wording but in the agent trajectory they encourage: \texttt{decompose} delays final SQL generation until the query logic has been broken into verified subcomponents, whereas \texttt{direct\_coder} pushes the agent to draft early and rely on execution feedback for rapid repair.
\begin{tcbraster}[raster columns=2, raster column skip=6pt, raster equal height=rows,
  enhanced, sharp corners, boxrule=0.4pt,
  colback=blue!3, colframe=blue!35,
  colbacktitle=blue!12, coltitle=black,
  fonttitle=\bfseries\ttfamily\footnotesize,
  fontupper=\ttfamily\footnotesize,
  left=5pt, right=5pt, top=3pt, bottom=3pt,
  attach boxed title to top left={xshift=4pt, yshift=-2pt},
  boxed title style={colback=blue!12, colframe=blue!12, sharp corners, boxrule=0pt}]
\begin{tcolorbox}[title={decompose skill}]
Break complex questions into simple subqueries, build bottom-up.\\
1. PARSE the question into atomic requirements.\\
2. BUILD each piece as a standalone \textcolor{sqlkw}{CTE}.\\
3. COMPOSE \textcolor{sqlkw}{CTE}s into the final query using \textcolor{sqlkw}{WITH...SELECT}.
\end{tcolorbox}
\begin{tcolorbox}[title={direct\_coder skill}]
You are an EFFICIENT \textcolor{sqlkw}{SQL} writer. Write \textcolor{sqlkw}{SQL} quickly, test, iterate.\\
1. Read the question carefully. Identify the core tables, joins, and aggregations.\\
2. Write your best \textcolor{sqlkw}{SQL} attempt IMMEDIATELY based on the schema.\\
3. Execute it. If errors occur, read the error message carefully and fix it.
\end{tcolorbox}
\end{tcbraster}
\end{example}


\vspace{-4mm}
\paragraph{Notation.}
Let $\mathcal{X}$ denote the space of input tasks $(q, \mathcal{D})$ and let $P$ be the underlying task distribution. For a skill $s \in \mathcal{S}$, we write $p_s(x) \in [0,1]$ for the probability that one execution of $a_s$ on input $x \in \mathcal{X}$ produces a correct SQL query (i.e., a query whose execution result matches the gold reference). For a finite training set $D_{\mathrm{train}} \subseteq \mathcal{X}$ and a subset $R \subseteq D_{\mathrm{train}}$, we write

\vspace{-4mm}
\begin{equation*}\small
\hat{p}_s(R) \;=\; \tfrac{1}{|R|}\sum\nolimits_{x \in R} p_s(x)
\end{equation*}
\vspace{-3mm}

for the empirical success rate of skill $s$ on $R$. 
For a collection of $K$ skills $A = \{s_1, \ldots, s_K\}$, the population $\operatorname{Pass@K}$, defined as the probability that at least one of the $K$ corresponding agent executions succeeds, is

\vspace{-4mm}
\begin{equation*}\small
\operatorname{Pass@K}(A) \;=\; \mathbb{E}_{x \sim P}\!\left[1 - \prod\nolimits_{j=1}^{K}\bigl(1 - p_{s_j}(x)\bigr)\right].
\end{equation*}
\vspace{-3mm}

\vspace{-1mm}
\subsection{Residual Skill Optimization}\label{sec:method:residual}

\vspace{-2mm}
\paragraph{The residual principle.}
As demonstrated in \Cref{fig:diagram}, we construct an ensemble of $K$ skills sequentially. After selecting skills $s_1, \ldots, s_{j-1}$, we define the residual training set

\vspace{-4mm}
\begin{equation*}\small
R_{j-1} \;=\; \bigl\{\, x_i \in D_{\mathrm{train}} : a_{s_\ell}\text{ fails on } x_i \;\;\forall \ell < j \,\bigr\},
\end{equation*}
and pick the next skill by maximizing success on this residual:

\vspace{-4mm}
\begin{equation*}\small
s_j \;\in\; \arg\max\nolimits_{s \in \mathcal{S}} \;\hat{p}_s(R_{j-1}).
\end{equation*}
\vspace{-4mm}

Later skills are therefore not pushed to be globally better than earlier ones. They are pushed to cover examples the current ensemble misses---which is exactly the marginal contribution of a new skill to $\operatorname{Pass@K}$. This is the mechanism by which the procedure encourages complementary skills and directly targets ensemble coverage rather than average accuracy.

\vspace{-3mm}
\paragraph{The \sys algorithm.}
The residual arg-max above is an idealized objective: it assumes access to the full training distribution and optimization over the infinite space of natural-language skills. \sys turns this principle into a practical batch-sequential process presented in Algorithm~\ref{alg:sys}. 
At each of $T$ rounds, \sys\ draws a fresh batch $B_t \subseteq D_{\mathrm{train}}$ and performs one pass of the residual principle over the seed pool: 
skills are drawn from $\mathcal{S}_0$ in randomized ordering; after each per-skill $\operatorname{SkillOptimizer}$ call, examples newly solved by the updated skill are removed from the residual set; and at the end of the batch, the accepted prompt updates are committed back to the seed pool, thus the pool evolves from batch to batch.
The algorithm has two main ingredients that make the ideal residual arg-max practical: a finite set of \emph{diverse seed skills} that defines the initial search space, and an \emph{inner-loop optimizer} that refines each seed on the current residual failures.

\begin{algorithm}[t]\small
\caption{\sys: batch-sequential residual skill optimization.}
\label{alg:sys}
\begin{algorithmic}[1]
\Require training question pool $\mathcal{D}_\mathrm{train}$; seed skill pool $\mathcal{S}_0 = \{s_1, \ldots, s_K\}$ with prompts $\{\pi_s^{(0)}\}_{s \in \mathcal{S}_0}$; batch size $b$; number of batches $T$; skill optimizer $\operatorname{SkillOptimizer}(\pi, R)$ that returns a refined prompt.
\State $\pi_s \gets \pi_s^{(0)}$ for each $s \in \mathcal{S}_0$ \Comment{current seed pool}
\For{$t = 1, \ldots, T$}
    \State sample a batch $B_t \subseteq \mathcal{D}_\mathrm{train}$ with $|B_t| = b$ \Comment{draw from question pool}
    \State $R_{t,0} \gets B_t$ \Comment{initial residual}
    \State choose an ordering $\sigma_t : [K] \to \mathcal{S}_0$ of the seed pool \Comment{without replacement}
    \For{$j = 1, \ldots, K$}
        \State $s \gets \sigma_t(j)$ \Comment{draw next seed}
        \State $\widetilde{\pi}_s \gets \operatorname{SkillOptimizer}(\pi_s; R_{t,j-1})$ \Comment{optimize on residual}
        \If{$\hat{p}_{\widetilde{\pi}_s}(R_{t,j-1}) > \hat{p}_{\pi_s}(R_{t,j-1})$}
            \State $\pi_s \gets \widetilde{\pi}_s$ \Comment{accept update on the residual}
        \EndIf
        \State $R_{t,j} \gets \{\, x \in R_{t,j-1} : a_s \text{ fails on } x \,\}$ \Comment{update residual with new $\pi_s$}
    \EndFor
    \State $\pi_s^{(t)} \gets \pi_s$ for each $s \in \mathcal{S}_0$ \Comment{commit seed-pool update}
\EndFor
\State \Return $\bigl\{ \pi_s^{(T)} \bigr\}_{s \in \mathcal{S}_0}$
\end{algorithmic}
\end{algorithm}

\vspace{-3mm}
\paragraph{Skill seed pool.}
The arg-max in $\arg\max_{s \in \mathcal{S}} \hat{p}_s(R_{j-1})$ is over an infinite, unstructured space of natural-language prompts and is intractable directly. To make the inner optimization tractable, we initialize from a small \emph{seed pool} $\mathcal{S}_0 = \{s_1^{(0)}, \ldots, s_K^{(0)}\}$ of $K$ 
LLM-assisted, manually curated skills,
each encoding a distinct high-level reasoning strategy.
Concretely, we first run a generic LLM agent on a subset of the training data and inspect representative and recurring failure modes, and propose diverse strategy prompts from several perspectives: whether to decompose the question before coding, how much schema and value exploration to perform before drafting SQL, etc. We inspect and keep only strategies with distinct intended trajectories, removing near-duplicates that differ only in wording. 
\Cref{ex:seed-skills} presents two simple, high-level skills we adopted as seeds, and the comprehensive list of seed skills is provided in Appendix~\ref{app:skill-pool}.
Subsequent rounds optimize prompts \emph{starting from} the seed pool rather than from scratch, restricting the effective search space to natural refinements of these high-level strategies.

\vspace{-3mm}
\paragraph{Inner-loop optimization.}
The inner loop uses an LLM to propose refined skill prompts based on observed agent failures, a standard reflective prompt-optimization step~\citep{fernando2023promptbreeder,zhou2022large,yang2023large, agrawal2025gepa}. At each round, the LLM is shown failure traces from running the current skill prompt $\pi_s$ on the residual set $R_{t,j-1}$, proposes a refined prompt $\widetilde{\pi}_s$ that addresses those failures, and the refinement is accepted if it improves recovery on $R_{t,j-1}$. The agentic setting introduces non-trivial structure: each failure trace contains a full sequence of schema inspections, intermediate query executions, and repair attempts, providing rich diagnostic signal but also requiring the optimizer to attribute blame across multi-step trajectories rather than to a single output. The \sys-specific element is the choice of training subset: by optimizing each skill against the residual rather than the full training set, refinements are biased toward recovering prior-ensemble failures rather than improving average accuracy on already-solved examples. 

The following example of the \texttt{decompose} skill evolution illustrates how residual optimization turns diversity in the seed pool into more targeted coverage in the inner loop optimization.

\vspace{-1mm}
\begin{example}[Decompose Skill Evolution]\label{ex:evolution-decompose-skill}
The \texttt{decompose} seed skill in \Cref{ex:seed-skills} already differs from skills such as \texttt{direct\_coder}: it encourages the agent to solve a query by breaking it into subproblems rather than drafting SQL immediately and repairing it based on execution feedback. 
However, the seed remains too high-level to specify what the agent should verify during decomposition. 
When optimized on the residual set, the prompt is refined precisely around the frequent kinds of errors that other skills leave unresolved. 
The example below shows the optimized \texttt{decompose} skill.

\vspace{-2mm}
\begin{tcolorbox}[
  enhanced, sharp corners, boxrule=0.4pt,
  colback=blue!3, colframe=blue!35,
  colbacktitle=blue!12, coltitle=black,
  fonttitle=\bfseries\ttfamily\footnotesize,
  fontupper=\ttfamily\footnotesize,
  title={optimized decompose skill},
  attach boxed title to top left={xshift=4pt, yshift=-2pt},
  boxed title style={colback=blue!12, colframe=blue!12, sharp corners, boxrule=0pt},
  left=5pt, right=5pt, top=3pt, bottom=3pt]
Break complex questions into simple subqueries, build bottom-up.\\[3pt]
\textbf{Step 1: PARSE the question into atomic requirements}\\
- What is the output? columns, derived metrics, ratios, counts, sums.\\
- What is the grain? one row per what? month? patient? group?\\
- What grouping dimensions are needed? only those that match the desired grain.\\
- Is there ranking, ordering, limiting, or a ratio/composition calculation?\\[3pt]
\textbf{Step 2: ANCHOR the grain before grouping}\\
- Match \textcolor{sqlkw}{GROUP BY} columns precisely to the output grain: no more, no less.\\
- If the question asks for monthly totals, group by month only; do not add\\
\hspace*{1.2em}route, city, or other columns unless explicitly requested.\\
- If the question asks for a split, identify the correct column and values that\\
\hspace*{1.2em}represent that split; do not substitute a loosely related column.\\[3pt]
\textbf{Step 3: BUILD each piece as a standalone \textcolor{sqlkw}{CTE}}\\
- Add joins one at a time, verifying row counts do not explode.\\
- For ratio/composition queries, compute totals in one \textcolor{sqlkw}{CTE}, subgroup counts in\\
\hspace*{1.2em}another, then \textcolor{sqlkw}{JOIN} and divide.\\[3pt]
\textbf{Step 4: VERIFY join logic and filter semantics}\\
- Confirm join keys actually link the intended entities; avoid fan-out.\\
- Confirm filter values match the domain as they appear in the data.
\end{tcolorbox}
\vspace{-1mm}
After optimization, it evolves from generic advice to ``build CTEs'' into a concrete rubric: identify the intended output grain before grouping, match \texttt{GROUP BY} columns to that grain, choose the correct column for requested splits, and compute ratio/composition queries using compatible aggregation levels. Thus, the learned skill is not merely a more detailed version of the seed prompt; its added details are shaped by the residual cases it is meant to cover. The full before/after skill comparison pool is provided in Appendix~\ref{app:skill-pool}.
\end{example}
\vspace{-2mm}

In practice, we further improve the finite-batch learning procedure through careful reflection prompt and reward design, rotating the skill order, etc.; Appendix~\ref{app:impl:practices} discusses these implementation practices in detail.



\vspace{-2mm}
\paragraph{Population-level guarantee.}
Under the population-level objective, residual skill optimization is greedy maximization of the Pass@K coverage objective. Formally, Proposition~\ref{proposition:pass-at-k} (stated and proved in Appendix~\ref{app:proposition-pass-at-k-proof}) shows that, in the population limit,

\vspace{-5mm}
\begin{equation*}\small
\operatorname{Pass@K}(\{s_1,\ldots,s_K\})
\;\geq\;
(1-1/e)\max_{|A|\le K}\operatorname{Pass@K}(A).
\end{equation*}
\vspace{-4mm}

Thus, the learned skill bank is guaranteed to achieve at least a constant-factor fraction of the best possible \(K\)-skill ensemble under the population objective. The intuition is that the marginal value of a new skill lies in its ability to solve problems that the current skill set still fails to address. Residual optimization therefore greedily adds the skill with the largest additional coverage of the remaining failure, yielding the standard approximation guarantee for monotone submodular maximization~\citep{nemhauser1978analysis}.

\vspace{-3mm}
\subsection{Inference}\label{sec:method:inference}
\vspace{-3mm}
At inference time, the $K$ learned skill-conditioned agents run in parallel on each test instance, producing $K$ candidate SQL queries. We select the final query using pairwise candidate comparisons following \citet{pourreza2024chase}. 
Since pairwise comparison scales quadratically in $K$, we first deduplicate candidates by execution output: queries returning identical results on the target database are collapsed into one equivalence class, with one representative retained. This leaves $G \le K$ candidates and reduces the number of comparisons when multiple skills agree. 
We then run an exhaustive round-robin over all $\binom{G}{2}$ unordered pairs, rather than sampling pairs, because $G$ is small in practice.\footnote{With $K=8$, $G\approx 1.7$ for baselines and $2.6$ for \sys on BIRD-Critic.} 
To mitigate LLM judge position bias, each pair $(i,j)$ is judged twice with swapped presentation order. Each judgment gives one win to the selected candidate, and we return the candidate with the highest win count, breaking ties arbitrarily.

\vspace{-3mm}
\section{Experiments}\label{sec:exp}
\vspace{-3mm}

In the experiments, we evaluate \sys through three research questions:

\vspace{-4mm}
\newtcolorbox{RQBox}{
    enhanced,
    colback=gray!5,       
    colframe=blue!40!black, 
    arc=0mm,                
    boxrule=0pt,            
    leftrule=3pt,           
    fonttitle=\bfseries\sffamily,
    coltitle=blue!40!black,
    attach title to upper,
    after title={\medskip\hrule\medskip}, 
    left=4mm,               
    right=4mm,
    top=2mm,
    bottom=2mm,
    before skip=15pt,
    after skip=15pt
}

\begin{RQBox}
    \textbf{RQ1: End-to-end effectiveness.} How does \sys compare with state-of-the-art Text-to-SQL systems on recent complex benchmarks?
    
    \textbf{RQ2: Transferability.} Do skills optimized on a single dataset or SQL dialect transfer to unseen datasets, alternative dialects, and new task formats?
    
    \textbf{RQ3: Behavioral diversity.} How does residual skill optimization change the behavior of an agentic Text-to-SQL system, beyond simply changing final SQL strings?
\end{RQBox}
\vspace{-3mm}

\begin{table}
\centering
\footnotesize
\setlength{\tabcolsep}{3pt}
\vspace{-4mm}
\begin{tabular}{p{3.5cm}p{5.3cm}}
\toprule
Seed skill & Core idea \\
\midrule
\texttt{default} & Balanced exploration and testing \\
\texttt{explore\_heavy} & Thorough data profiling before drafting \\
\texttt{direct\_coder} & Quick draft, and refine incrementally \\
\texttt{decompose} & Build from validated substeps like joins \\
\texttt{conservative} & Use the simplest faithful query \\
\texttt{adversarial\_checker} & Stress-test joins, filters, grouping, etc. \\
\texttt{template\_first} & Start from templates and adapt \\
\texttt{fast\_error\_repair} & Repair concrete errors one at a time \\
\bottomrule
\end{tabular}
\caption{Summary of initial seed skills.}
\vspace{-8mm}
\label{tab:seed-pool}
\end{table}

\vspace{-3mm}
\subsection{Experimental Setup}\label{exp:setup}



\vspace{-2mm}
\paragraph{Benchmarks.}
We evaluate \sys on two recent Text-to-SQL benchmarks that provide clean ground-truth annotations and involve complex reasoning. First, we use Spider2-Lite~\citep{lei2024spider}, which tests complex SQL generation over realistic schemas and multiple dialects. We report results on its SQLite, Snowflake, and BigQuery subsets, with 135, 207, and 209 examples respectively. These subsets differ in schema organization, query style, and dialect-specific syntax, allowing us to study both in-domain performance and cross-dialect transfer.
Second, we evaluate on BIRD-Critic~\citep{li2025swe}, a SQL debugging benchmark derived from BIRD~\citep{li2024can}. Each instance provides a natural-language issue description, a buggy SQL query, and database context; the system must diagnose and repair the query rather than generate SQL from scratch. This setting tests whether \sys transfers beyond direct SQL generation to agentic SQL correction. We evaluate on its pure PostgreSQL version.

\vspace{-2mm}
\paragraph{Skill optimization.}
We adopt the state-of-the-art prompt and skill optimization technique GEPA~\citep{agrawal2025gepa} for skill optimization.
For the Spider2-Lite benchmark, we learn skills from only approximately 200 examples sampled from proprietary data in Snowflake SQL dialect, and then evaluate the learned skills on Spider2-Lite. 
For the BIRD-Critic benchmark, we optimize skills on BIRD-mini-dev~\cite{li2024can}, which contains 500 standard Text-to-SQL examples. 


\vspace{-2mm}
\paragraph{Baselines.}
We compare \sys against representative open-source or most relevant Text-to-SQL systems, including \dinsql~\citep{pourreza2023din},
\reforce~\citep{deng2025reforce}, and \chasesql~\citep{pourreza2024chase}. 
We use the open-source implementations of \dinsql and \reforce directly.
Since \chasesql was originally designed for a workflow-based pipeline, to adapt it to our agentic setting, we retain its transferable design choices, including schema-link shuffling, high-temperature sampling, and pairwise candidate selection, while replacing the fixed, manually designed workflow with our agent architecture\footnote{In our experiments, the workflow-based solution consistently underperforms the agentic solution unless heavy engineering effort is applied, which led us to build around the agentic approach.}.
By default, we use the Opus-4.6\cite{anthropic2026opus46} model. However, we also show results with GPT-5.4~\cite{openai2026gpt54}.
We share the implementation details in \Cref{app:impl}.

\vspace{-2mm}
\paragraph{Metrics.}
We report three metrics. \emph{Pass@1} measures the execution accuracy of a single generated candidate. \emph{Pass@8} measures the oracle candidate-set accuracy: a problem is counted as solved if at least one of the eight generated candidates is correct. Pass@8 therefore measures the quality and coverage of candidate generation independently of selection\footnote{We compute Pass@8 based on internal candidates of ensemble-based methods. As \reforce is generating a dynamic number of candidates in its workflow, we are not able to compute Pass@8.}.
Finally, \emph{selected accuracy} measures the execution accuracy of the single SQL query returned by the selector, either through LLM judge or majority voting, and is the end-to-end performance of the deployed system. 
This applies to ensemble-based methods including \sys, \reforce, and \chasesql.

\paragraph{Skill seed pool and choice of $K$.}
We observe the Pass@$k$ curve for various methods nearly saturates beyond $K{=}8$ empirically. Therefore, unless explicitly specified, we use $K{=}8$ by default in the experiments.
We initialize the optimization with $K{=}8$ seed skills (see seed construction in \Cref{sec:method}), each representing a distinct agent behavior family.
Table~\ref{tab:seed-pool} lists all seeds with a one-line strategy summary; full prompt text is given in Appendix~\ref{app:skill-pool}.


\vspace{-3mm}
\subsection{End-to-End Performance and Transferability}
\vspace{-2mm}
\begin{table*}[t]
\centering
\small
\setlength{\tabcolsep}{3pt}

\begin{subtable}{\textwidth}
\centering
\begin{tabular}{lccccccccc}
\toprule
& \multicolumn{3}{c}{SQLite} & \multicolumn{3}{c}{Snowflake} & \multicolumn{3}{c}{BigQuery} \\
\cmidrule(lr){2-4}\cmidrule(lr){5-7}\cmidrule(lr){8-10}
Method & pass@1 & pass@8 & Sel. acc. & pass@1 & pass@8 & Sel. acc. & pass@1 & pass@8 & Sel. acc. \\
\midrule
\dinsql & 40.74 & / & / & 0.97 & / & / & 18.54 & / & / \\
\reforce & 55.28 & / & 57.78 & 43.47 & / & 50.72 & 51.10 & / & 55.12 \\
\chasesql & 62.59 & 76.30 & 63.70 & 51.21 & 68.60 & 53.14 & 59.02 & 71.71 & 56.59 \\
\rowcolor{lightblue!50}
\sys & 62.13 & 73.33 & \textbf{64.44} & 60.08 & 72.46 & \textbf{64.25} & 62.07 & 73.17 & \textbf{64.88} \\
\bottomrule
\end{tabular}
\caption{}
\label{tab:exp:spider2lite-opus}
\end{subtable}

\vspace{1mm}

\begin{subtable}{\textwidth}
\centering
\begin{tabular}{lccccccccc}
\toprule
& \multicolumn{3}{c}{SQLite} & \multicolumn{3}{c}{Snowflake} & \multicolumn{3}{c}{BigQuery} \\
\cmidrule(lr){2-4}\cmidrule(lr){5-7}\cmidrule(lr){8-10}
Method & pass@1 & pass@8 & Sel. acc. & pass@1 & pass@8 & Sel. acc. & pass@1 & pass@8 & Sel. acc. \\
\midrule
\dinsql & 37.04 & / & / & 0 & / & / & 2.93 & / & / \\
\reforce & 64.44 & / & 58.52 & 47.83 & / & 41.06 & 48.29 & / & 51.22 \\
\chasesql & 61.57 & 83.70 & 66.67 & 51.69 & 73.91 & 57.97 & 53.90 & 74.15 & 59.51 \\
\rowcolor{lightblue!50}
\sys & 62.87 & 84.44 & \textbf{71.85} & 52.23 & 76.33 & \textbf{61.84} & 57.80 & 77.56 & \textbf{63.41} \\
\bottomrule
\end{tabular}
\caption{}
\label{tab:exp:spider2lite-gpt}
\end{subtable}

\vspace{-1mm}
\caption{Spider2-Lite results across SQLite, Snowflake, and BigQuery using (a) Opus 4.6 and (b) GPT-5.4.}
\label{tab:exp:spider2lite}
\vspace{-2mm}
\end{table*}

\begin{table*}[t]
\centering
\small
\setlength{\tabcolsep}{4pt}
\begin{tabular}{lccc}
\toprule
Hallucination Type & \chasesql & \sys & Ratio \\
\midrule
Pools with invalid-reference candidate & 10 & \textbf{7} & 1.43x \\
Solvable pools with invalid-reference candidate & 6 & \textbf{2} & 3.00x \\
Invalid-reference candidate slots & 19 & \textbf{13} & 1.46x \\
Missing-function hallucination cases & 6 & \textbf{2} & 3.00x \\
\bottomrule
\end{tabular}
\vspace{-1mm}
\caption{Hallucination diagnostics on Snowflake instances based on invalid-reference failures.}
\label{tab:chasesql-gpt54-snowflake-hallucination}
\vspace{-4mm}
\end{table*}

\begin{table*}[t]
\centering
\small
\setlength{\tabcolsep}{4pt}
\begin{tabular}{lccc}
\toprule
Structural Mismatch vs.\ Gold SQL & \chasesql & \sys & Ratio \\
\midrule
Wrong \textsc{distinct} usage & 15 & \textbf{10} & 1.50x \\
Wrong window-function usage & 6 & \textbf{2} & 3.00x \\
Wrong \textsc{union} structure & 6 & \textbf{3} & 2.00x \\
\bottomrule
\end{tabular}
\vspace{-1mm}
\caption{Structural mismatch analysis on Snowflake instances with reference SQL available.}
\label{tab:chasesql-gpt54-snowflake-gold-followup}
\vspace{-3mm}
\end{table*}

\Cref{tab:exp:spider2lite-opus,tab:exp:spider2lite-gpt,tab:bird-critic:pg} report end-to-end results on BIRD-Critic and Spider2-Lite, grouped by dialect. 
Overall, \sys achieves the strongest selected accuracy in almost all settings and consistently outperforms the strongest ensemble baseline, \chasesql.

The advantage is most visible on the harder dialects in Table~\ref{tab:exp:spider2lite-opus}. 
With Opus-4.6, \sys improves selected accuracy over \chasesql by +11.11 on Snowflake and +8.29 on BigQuery, where large schemas and dialect-specific syntax make unstable decoding more costly. 
On SQLite, \sys has slightly lower Pass@8 than \chasesql (73.33 vs.\ 76.30), but still achieves higher selected accuracy (64.44 vs.\ 63.70). 
This suggests that raw coverage alone is insufficient: SQLite is easier, so stochastic sampling can already cover many cases, but the additional candidates may include plausible wrong queries that are harder for the selector to distinguish~\cite{pourreza2025reasoning}. 
The same pattern holds with GPT-5.4 in Table~\ref{tab:exp:spider2lite-gpt}, where \sys improves both Pass@8 and selected accuracy across all three dialects.
Note that \reforce selects its final SQL by majority voting over internal candidates. 
With GPT-5.4, its selected result underperforms mean Pass@1 on SQLite and Snowflake, showing that majority voting can fail when many candidates converge to the same incorrect result. 
See Appendix~\ref{app:additional-exp} for detailed analysis.

\vspace{-3mm}
\paragraph{Hallucination and Error Analysis.}\Cref{tab:chasesql-gpt54-snowflake-hallucination,tab:chasesql-gpt54-snowflake-gold-followup} provide a more structured error breakdown, evaluated on the Snowflake part of Spider2-Lite. 
Compared with \chasesql, \sys produces fewer invalid-reference failures: the number of pools containing such a candidate drops from 10 to 7, and among solvable pools, the count drops from 6 to 2. 
Missing-function hallucinations also decrease from 6 to 2. The structural comparison against gold SQL shows a similar pattern: \sys makes fewer errors involving \textsc{distinct}, window functions, and \textsc{union} structure. 
These results indicate that residual skills \emph{improve diversity in a more controlled way}. Rather than relying on high-temperature perturbations that can introduce hallucinated references or unstable SQL structures, \sys induces different agent behaviors while preserving candidate quality.

\vspace{-3mm}
\paragraph{Transfer across SQL dialects and task settings.}
We next analyze whether skills encode reusable problem-solving strategies or merely overfit to one dialect or task. 



Recall that on Spider2-Lite, the skills are optimized only on Snowflake SQL format data, but are applied to all dialects without further optimization. 
The end-to-end results in \Cref{tab:exp:spider2lite-opus,tab:exp:spider2lite-gpt} show that \sys outperforms all baselines on SQLite and BigQuery by a large margin. 
This suggests that residual skill optimization does not merely memorize Snowflake-specific syntax. Instead, as we will show later in \Cref{exp:trajectory}, the learned skills capture higher-level strategies for schema exploration, decomposition, query construction, and error checking that transfer across dialects.

\begin{wraptable}[7]{r}{0.4\textwidth}
    \centering
    \footnotesize
    \setlength{\tabcolsep}{4pt}
    \vspace{-4mm}
    \begin{tabular}{lccc}
        \toprule
        \textbf{Method} & \textbf{pass@1} & \textbf{pass@8} & \textbf{Sel. acc.} \\
        \midrule
        \dinsql   &  30.38    & --    & --    \\
        \reforce  & 41.50 & -- & 43.21 \\
        \chasesql &  44.69 & 52.26  & 46.23  \\
        \rowcolor{lightblue!50}
        \sys      & 46.16  & 54.53  &  \textbf{48.87} \\
        \bottomrule
    \end{tabular}
    \vspace{-2mm}
    \caption{Bird-Critic PostgreSQL Results.}
    \label{tab:bird-critic:pg}

\end{wraptable}
Although all text-to-SQL systems, datasets, and benchmarks target accurate SQL generation given a user's question and a database, their settings vary widely. For instance, BIRD-Critic focuses on debugging SQL using feedback and a given buggy SQL query, which is quite different from existing benchmarks. In this case, finding the training data in the same format is challenging. 
In this experiment, we explore whether the skills learned on Bird-mini-dev---a regular Text-to-SQL dataset with no feedback or buggy SQL---translate well to Bird-Critic.
Results in Table~\ref{tab:bird-critic:pg} confirm that the skills transfer: \sys improves selected accuracy by $+2.64$ points over \chasesql, with consistent gains in pass@1 ($+1.47$) and pass@8 ($+2.27$). 
The simultaneous pass@8 improvement indicates that the learned skills broaden the candidate pool, even on a task format (debugging from feedback) they were not optimized for.

\begin{figure}[t]
    \centering
    \begin{subfigure}[b]{0.4\linewidth}
        \centering
        \includegraphics[width=\linewidth]{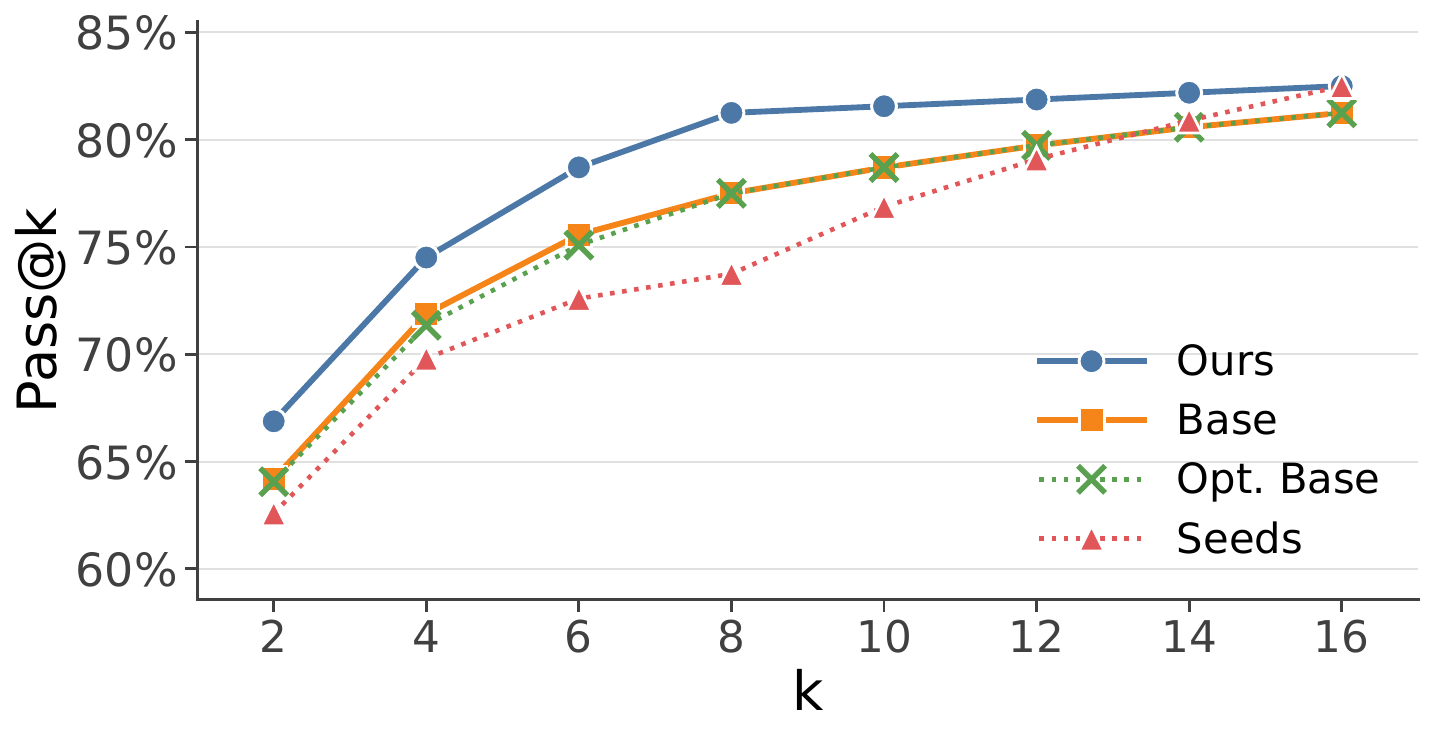}
        \vspace{-7mm}
        \caption{}
        \vspace{-3mm}
        \label{fig:passk:sf}
    \end{subfigure}
    \hspace{7mm}
    \begin{subfigure}[b]{0.4\linewidth}
        \centering
        \includegraphics[width=\linewidth]{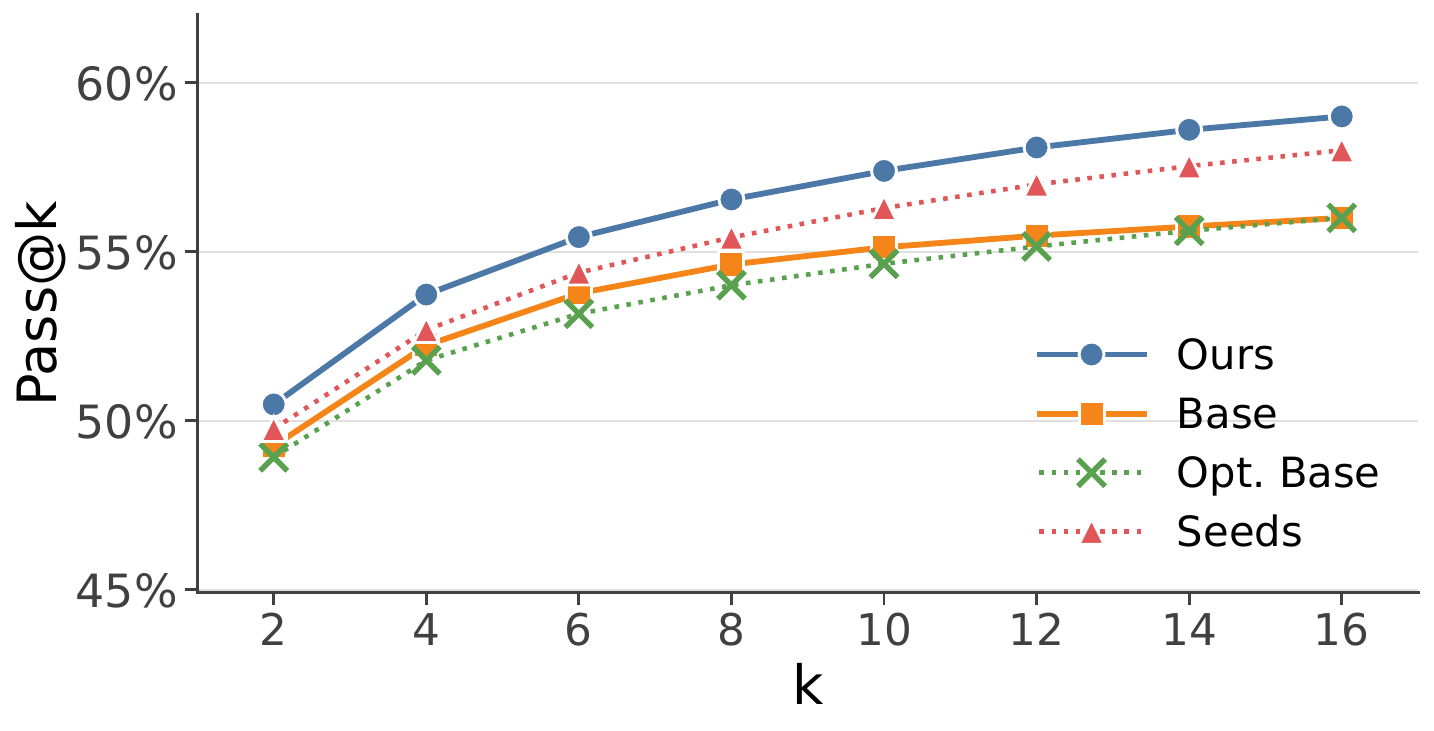}
        \vspace{-7mm}
        \caption{}
        \vspace{-3mm}
        \label{fig:passk:bird}
    \end{subfigure}
    \caption{Pass@k comparison between \sys and its variants on 100-instance subsets of a) Spider2-lite and b) Bird-Critic.}
    \label{fig:passk}
    \vspace{-3mm}
\end{figure}



\begin{figure*}[t]
    \centering

    \includegraphics[
        width=0.95\textwidth,
        height=0.26\textheight,
        keepaspectratio
    ]{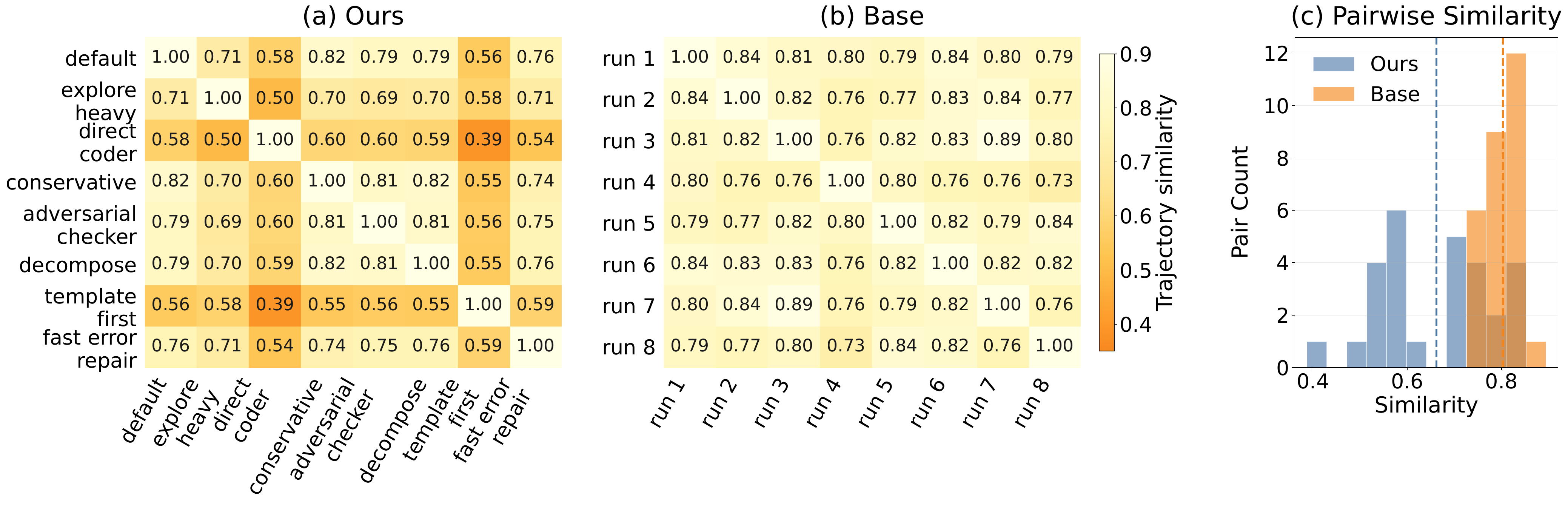}
    \vspace{-5mm}
    \caption{Trajectory comparison between \sys and repeated runs of the default skill on the Snowflake part of Spider2-Lite.}
    \label{fig:trajectory-comp-sf}


    \includegraphics[
        width=0.95\textwidth,
        height=0.26\textheight,
        keepaspectratio
    ]{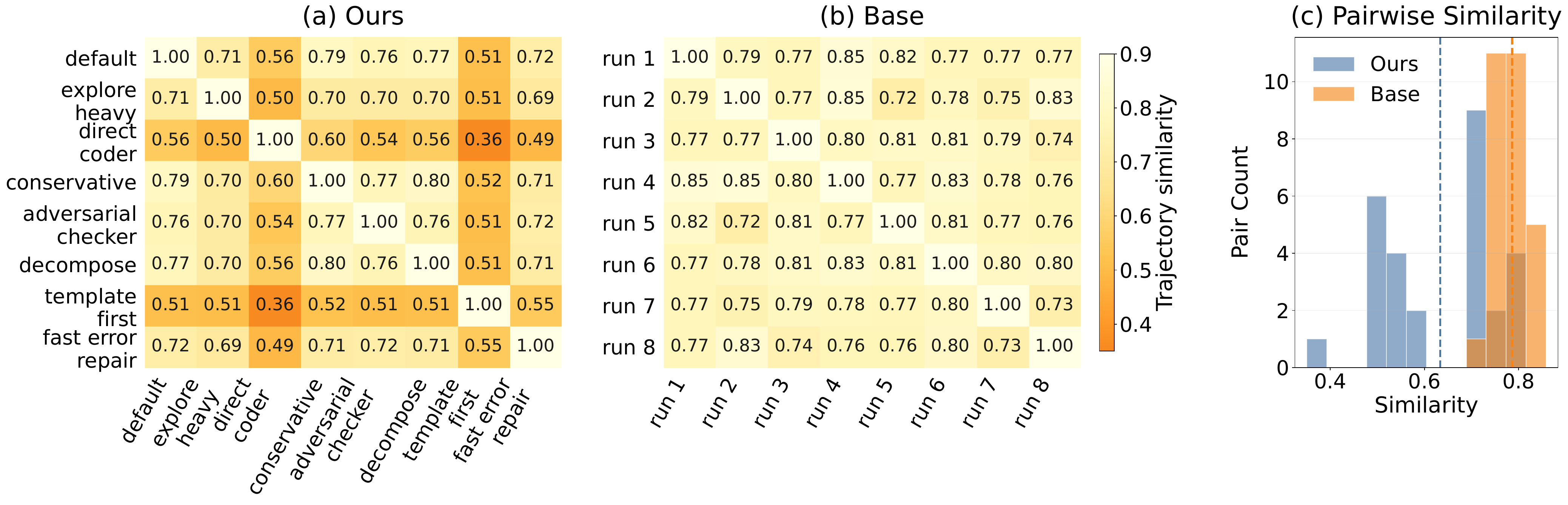}
    \vspace{-5mm}
    \caption{Trajectory comparison between \sys and repeated runs of the default skill on the BigQuery part of Spider2-Lite.}
    \label{fig:trajectory-comp-bq}


    \includegraphics[
        width=0.95\textwidth,
        height=0.26\textheight,
        keepaspectratio
    ]{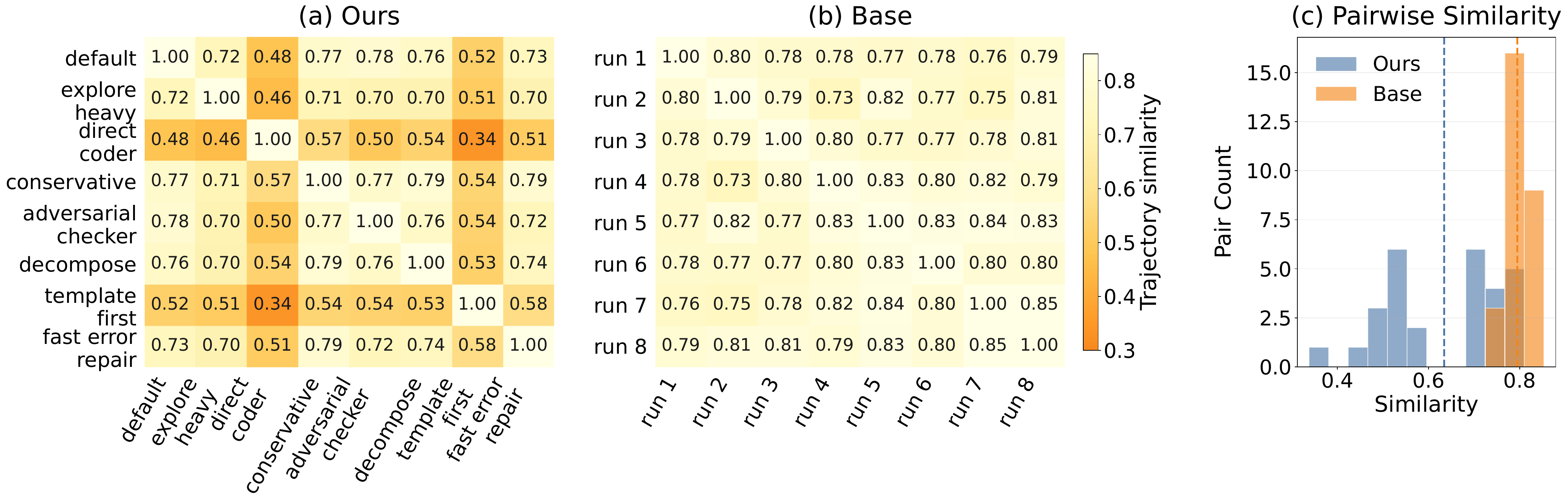}
    \vspace{-5mm}
    \caption{Trajectory comparison between \sys and repeated runs of the default skill on the SQLite part of Spider2-Lite.}
    \label{fig:trajectory-comp-sqlite}

\end{figure*}

\vspace{-3mm}
\subsection{Ablation Studies}
\vspace{-2mm}

To understand the effectiveness of each component of \sys, including the residual optimization process and the skill pool, we compare \sys with its three variants: \emph{Base}, which repeatedly samples from the original base skill; \emph{Opt. Base}, which runs GEPA to optimize residuals from only base skills rather than diverse seeds; and \emph{Seeds}, which uses the unoptimized initial skill seeds.
These baselines separate the effect of residual optimization from the effect of simply sampling more times, optimizing one stronger prompt, or using manually diverse initial instructions.


Figure~\ref{fig:passk} studies how candidate-set coverage changes as we increase the number of generated candidates. 
Repeated sampling from the base skill, i.e. Base, improves Pass@$k$ only gradually, indicating that independent runs of the same skill tend to make correlated mistakes. Optimizing a single base skill improves individual quality, but still leaves many residual failures uncovered. 
Initial diverse set of seed skills helps compared to Base in some cases (\Cref{fig:passk:bird}), but might also degrades performance (\Cref{fig:passk:sf}).

In contrast, \sys achieves the strongest Pass@$k$ curve across nearly all values of $k$. 
The gain is especially meaningful at small and moderate $k$, where each additional candidate must cover new failure modes to be useful. 
As a result, to achieve the same coverage as \sys's pass@8, baselines need 3 to 8 additional passes, making \sys a more cost-efficient choice.
This behavior is exactly what residual skill optimization is designed to produce: each later skill is not optimized to be globally better on all questions, but to solve examples that previous skills miss. 

\vspace{-3mm}
\subsection{How Diverse Skills Change Agent Behavior}\label{exp:trajectory}
\vspace{-3mm}

To understand why residual skill optimization changes Text-to-SQL agent behavior, we analyze agent trajectories rather than only final SQL outputs. 
A trajectory records the sequence of high-level actions taken by an agent, such as schema inspection, exploratory SQL generation, and error repair. 
We measure trajectory dissimilarity between two runs using edit distance normalized to $[0,1]$, and define trajectory similarity as $1$ minus this normalized distance. 
Lower similarity therefore indicates that two candidates are produced through more distinct reasoning and tool-use paths.



\Cref{fig:trajectory-comp-sf} shows a clear contrast between \sys and repeated sampling. Repeated runs form a high-similarity cluster, with most pairwise similarities concentrated around $0.75$--$0.85$, suggesting that sampling alone mostly produces variants of the same reasoning path. In contrast, \sys spreads trajectories over a much wider similarity range: learned skills such as direct-coder, template-first, decomposition, and exploration-heavy form low-similarity pairs, indicating that they induce genuinely different agent behaviors rather than merely different SQL surface forms. This pattern also transfers to the BigQuery and SQLite subsets of Spider2-Lite (\Cref{fig:trajectory-comp-bq,fig:trajectory-comp-sqlite}), where skills learned from Snowflake data are applied without retraining, suggesting that the behavioral changes are not tied to a single SQL dialect. Overall, the analysis supports the central mechanism of \sys: learned skills change how the agent approaches the task, thereby producing candidate sets with less-correlated failure modes and helping explain the stronger Pass@$K$ curve in \Cref{fig:passk}.

\vspace{-3mm}
\section{Discussion, Limitations, and Conclusions}\label{sec:conclusion}
\vspace{-3mm}

\paragraph{Discussion and limitations.}
\sys primarily improves candidate generation and does not focus on enhanced candidate selection method.
Across several settings, there remains a substantial gap between Pass@8 and selected accuracy, indicating that the correct SQL is often present in the candidate pool but not selected. Our pairwise comparison procedure mitigates direct 1-of-$K$ selection difficulty and reduces position bias by swapping candidate order, but it still relies on an LLM judge and scales quadratically in the number of surviving candidates. 
Future work could combine residual skill optimization with stronger selectors. 
Recent advances in language model reasoning demonstrate that verifiers trained on self-generated correct and incorrect trajectories—often formulated as outcome or process reward models—can significantly enhance test-time selection among multiple candidates \citep{cobbe2021training, lightman2023let, hosseini2024v, ni2023lever}. This paradigm is naturally aligned with our framework: \sys already generates diverse candidate reasoning paths and SQL programs, while execution feedback provides an abundant, automated source of positive and negative supervision to effectively train such selectors without human annotation.

In addition, our evaluation focuses on recent Text-to-SQL and SQL-debugging benchmarks with executable ground truth. Although the learned skills transfer across SQL dialects and even from standard Text-to-SQL training data to BIRD-Critic, broader deployment settings may introduce additional challenges, including ambiguous user intent, missing schema documentation, and multi-turn interactive feedback-based debugging. Future work might explore how to adapt \sys to these more complex settings.

\vspace{-3mm}
\paragraph{Conclusions.}
We introduce \sys, a residual skill optimization framework for constructing complementary Text-to-SQL ensembles. 
Rather than relying on hand-designed prompt or workflow variants, or high-temperature sampling, \sys learns a bank of skill-conditioned agents, where each skill is optimized to recover examples missed by the others. This directly targets candidate-set coverage and improves Pass@K while preserving the quality of SQL candidates. Across Spider2-Lite and BIRD-Critic, spanning multiple SQL dialects and task formats, \sys improves end-to-end accuracy over strong ensemble baselines while producing fewer hallucinated outputs and less redundant agent behavior.

\bibliographystyle{plainnat}  
\bibliography{reference}

\newpage
\appendix

\section{Proofs}
\label{app:proofs}
\subsection{Proof of Proposition~\ref{proposition:pass-at-k}}
\label{app:proposition-pass-at-k-proof}
We analyze the population-level version of residual skill optimization. In this setting, the next skill is chosen to maximize its expected contribution on the residual failure mass of the current skill bank. The finite-batch procedure in Algorithm~\ref{alg:sys} can be viewed as an empirical approximation to this objective; deriving finite-sample guarantees would require additional assumptions on sample size and generalization of the learned skills.
\begin{proposition}[Residual skill optimization approximates optimal Pass@K]
\label{proposition:pass-at-k}
Let \(\mathcal S\) be a fixed skill family, and let \(p_s(x)\in[0,1]\) denote the probability that one execution of skill \(s\) solves input \(x\). For any skill bank \(A\subseteq\mathcal S\), define its population Pass@K objective as
\[
F(A)
=
\mathbb E_{x\sim P}
\left[
1-\prod_{s\in A}(1-p_s(x))
\right].
\]
Starting from \(A_0=\emptyset\), suppose that at each round \(j=1,\ldots,K\), residual skill optimization selects
\[
s_j
\in
\arg\max_{s\in\mathcal S}
\mathbb E_{x\sim P}
\left[
p_s(x)\prod_{s'\in A_{j-1}}(1-p_{s'}(x))
\right],
\]
where \(A_{j-1}=\{s_1,\ldots,s_{j-1}\}\), and sets \(A_j=A_{j-1}\cup\{s_j\}\). Let
\[
A^\star\in\arg\max_{A\subseteq\mathcal S,\ |A|\le K} F(A)
\]
be the optimal size-\(K\) skill bank. Then
\[
F(A_K)
\ge
(1-1/e)F(A^\star).
\]
Equivalently,
\[
\operatorname{Pass@K}(\{s_1,\ldots,s_K\})
\ge
(1-1/e)
\max_{|A|\le K}\operatorname{Pass@K}(A).
\]
\end{proposition}
\begin{proof}
The proof follows the standard greedy analysis for monotone submodular maximization under a cardinality constraint~\citep{nemhauser1978analysis}; we first verify that the Pass@K objective in our setting is indeed monotone submodular.

We first show that the Pass@K objective is monotone submodular. For a skill bank
\(A\subseteq\mathcal S\), recall that
\[
F(A)
=
\mathbb E_{x\sim P}
\left[
1-\prod_{s\in A}(1-p_s(x))
\right].
\]
For any skill \(s\notin A\), the marginal gain of adding \(s\) is
\[
\Delta(s\mid A)
=
F(A\cup\{s\})-F(A).
\]
Expanding the definition of \(F\), we obtain
\[
\begin{aligned}
\Delta(s\mid A)
&=
\mathbb E_{x\sim P}
\left[
1-(1-p_s(x))\prod_{s'\in A}(1-p_{s'}(x))
\right]
-
\mathbb E_{x\sim P}
\left[
1-\prod_{s'\in A}(1-p_{s'}(x))
\right] \\
&=
\mathbb E_{x\sim P}
\left[
p_s(x)\prod_{s'\in A}(1-p_{s'}(x))
\right].
\end{aligned}
\]
Since \(p_s(x)\in[0,1]\), this marginal gain is nonnegative. Hence \(F\) is monotone.

Next, let \(A\subseteq B\subseteq\mathcal S\). Since each factor \(1-p_{s'}(x)\in[0,1]\), we have
\[
\prod_{s'\in B}(1-p_{s'}(x))
\le
\prod_{s'\in A}(1-p_{s'}(x)).
\]
Therefore,
\[
\Delta(s\mid B)
=
\mathbb E_{x\sim P}
\left[
p_s(x)\prod_{s'\in B}(1-p_{s'}(x))
\right]
\le
\mathbb E_{x\sim P}
\left[
p_s(x)\prod_{s'\in A}(1-p_{s'}(x))
\right]
=
\Delta(s\mid A).
\]
Thus \(F\) satisfies diminishing marginal returns and is submodular.

The residual arg-max rule is exactly greedy maximization of this objective, because the marginal gain of adding skill \(s\) to the current bank \(A_{j-1}\) is
\[
\Delta(s\mid A_{j-1})
=
\mathbb E_{x\sim P}
\left[
p_s(x)\prod_{s'\in A_{j-1}}(1-p_{s'}(x))
\right].
\]
Thus choosing \(s_j\) by residual maximization is the same as choosing the skill with largest marginal increase in \(F\).

Let \(A^\star\in\arg\max_{|A|\le K}F(A)\) be an optimal size-\(K\) skill bank. Since \(F\) is monotone,
\[
F(A^\star)-F(A_j)
\le
F(A_j\cup A^\star)-F(A_j).
\]
By submodularity,
\[
F(A_j\cup A^\star)-F(A_j)
\le
\sum_{s\in A^\star}\Delta(s\mid A_j).
\]
Since \(|A^\star|\le K\), at least one skill \(s\in A^\star\) has marginal gain at least
\[
\frac{F(A^\star)-F(A_j)}{K}.
\]
Greedy chooses the skill with largest marginal gain, so
\[
F(A_{j+1})-F(A_j)
\ge
\frac{F(A^\star)-F(A_j)}{K}.
\]
Equivalently,
\[
F(A^\star)-F(A_{j+1})
\le
\left(1-\frac{1}{K}\right)
\left(F(A^\star)-F(A_j)\right).
\]
Applying this recurrence for \(K\) rounds gives
\[
F(A^\star)-F(A_K)
\le
\left(1-\frac{1}{K}\right)^K F(A^\star).
\]
Hence
\[
F(A_K)
\ge
\left(1-\left(1-\frac{1}{K}\right)^K\right)F(A^\star)
\ge
(1-1/e)F(A^\star).
\]
This proves the claim.
\end{proof}
\section{Implementation Details and Best Empirical Practices}\label{app:impl}

\subsection{Implementation Details}\label{app:impl:impl}
\paragraph{Machine setup.}
All experiments were run on a single Apple Silicon MacBook Pro
(Apple M-series CPU, 10+ cores; 32\,GB unified memory; macOS~14).
Since our pipeline issues all language-model calls to hosted inference APIs, no local GPU is required: the workstation only orchestrates prompting, parsing, and result aggregation. 

\paragraph{LLM backbone and inference settings.}
We evaluate \sys with two LLM backbones: Opus-4.6 and GPT-5.4. Both are used in non-reasoning mode with a decoding temperature of $0.2$, a canonical choice for coding tasks, and a maximum completion budget of 64000 tokens. The low temperature is a deliberate design choice: since \sys achieves diversity through learned skills rather than stochastic decoding, a low temperature preserves the reasoning stability and precision needed for complex SQL generation. Each agent run is allowed up to 12 reasoning turns and 20 SQL executions.

\paragraph{\chasesql adaptation.}
\chasesql was originally designed for a workflow-based Text-to-SQL pipeline. To produce a fair comparison in our agentic setting, we adapt its three transferable design choices to our agent architecture:
\begin{itemize}[nosep,leftmargin=*]
    \item \emph{Schema-link shuffling.} We permute the column ordering presented to the agent across different candidates, matching \chasesql's schema perturbation strategy.
    \item \emph{High-temperature decoding.} We set the agent's decoding temperature to $1.0$ for \chasesql candidates, reproducing the stochastic variation that \chasesql uses to induce candidate diversity.
    \item \emph{Pairwise candidate selection.} We use the same LLM-based pairwise comparison selector for both \chasesql and \sys, ensuring the selection mechanism is held constant across methods.
\end{itemize}

\paragraph{Candidate selection.}
For both \sys and \chasesql, we use an LLM-based pairwise selector that compares each candidate pair and selects the winner via win-rate-based aggregation. The selector uses the same backbone model as the candidate generator (either Opus-4.6 or GPT-5.4) at a temperature of $0.2$. Each pairwise comparison receives the question, database schema, and the two candidate SQL queries along with execution previews, and returns a preference judgment.

\paragraph{GEPA optimization settings.}
Table~\ref{tab:gepa-hyperparams} summarizes the GEPA hyperparameters used for skill optimization.

\begin{table}[t]
\centering
\small
\setlength{\tabcolsep}{4pt}
\begin{tabular}{ll}
\toprule
Hyperparameter & Value \\
\midrule
\multicolumn{2}{l}{\emph{Agent (evaluation)}} \\
\quad Agent model & GPT-5.4 ($T{=}0.2$) or Sonnet-4.6 \\
\quad Max agent turns & 12 \\
\quad Max SQL executions & 20 \\
\quad Max completion tokens & 64000 \\
\midrule
\multicolumn{2}{l}{\emph{GEPA optimizer}} \\
\quad Reflection model & GPT-5.4 ($T{=}1.0$) or Opus-4.6 \\
\quad Task proposal temperature & 0.2 \\
\quad Reflection max tokens & 64000 \\
\quad Metric calls per skill per batch & 20 \\
\quad Reflection minibatch size & 3 \\
\midrule
\multicolumn{2}{l}{\emph{Training loop}} \\
\quad Number of batches & 3 \\
\quad Batch size (train / val) & 70 / 30 \\
\quad Held-out validation size & 80 \\
\quad Max prompt length & 12000 chars \\
\bottomrule
\end{tabular}
\caption{GEPA hyperparameters for BIRD-mini-dev skill optimization. Spider2-Lite Snowflake training uses the same settings except the training data source.}
\label{tab:gepa-hyperparams}
\end{table}

\paragraph{Agent tools.}
The agent has access to six tools during both training and evaluation:
\begin{itemize}[nosep,leftmargin=*]
    \item \texttt{execute\_sql}: run SQL against the live database and return results or error messages.
    \item \texttt{lookup\_docs}: retrieve dialect-specific documentation (e.g., dialect function/grammar reference), database meta-data, external knowledge, etc.
    \item \texttt{review\_sql}: invoke an LLM-based critic to review the current SQL draft before submission.
    \item \texttt{get\_sql\_pattern}: retrieve anonymized SQL patterns for similar query types (e.g., top-N, running totals).
    \item \texttt{get\_sql\_templates}: retrieve SQL templates that are masked training instances categorized by query types.
    \item \texttt{submit\_final\_sql}: submit the final answer.
\end{itemize}
The tool set and their implementations remain fixed across all skills and all experiments. Skills influence only the strategy text in the system prompt; they cannot add, remove, or modify tools.

\subsection{Practices for Better Skill Learning}\label{app:impl:practices}

We describe five practical lessons that proved important for making reflective skill optimization work reliably on Text-to-SQL agents.

\paragraph{Proxy models for optimization.}
Running the full GEPA reflect-mutate-evaluate loop on the strongest available model (e.g., Opus~4.6) is both expensive and, perhaps counter-intuitively, less effective.
Stronger models already produce near-correct trajectories on many training examples, leaving the reflector with subtle failure signals that are hard to attribute to specific strategy gaps.
We find that using a weaker model from the same family as a proxy, e.g., Sonnet~4.6 for the Opus~4.6 experiments, yields faster and cheaper optimization while producing larger per-round accuracy gains that push skill evolution more aggressively.
The resulting optimized skills transfer well to the stronger model: the strategy-level improvements (e.g., ``anchor the grain before grouping'') are model-agnostic, even though they were discovered from the proxy's more frequent and more interpretable failures.

\paragraph{Brevity as a tiebreaker.}
During GEPA optimization, multiple candidate skill mutations often achieve the same accuracy gain on the hard batch.
Rather than selecting arbitrarily, we break ties by preferring the shorter prompt.
This acts as a lightweight regularizer: longer prompts tend to accumulate instance-specific details from the training batch, increasing the risk of overfitting to particular schemas or question patterns.

\paragraph{Dialect- and instance-agnostic reflection.}
We explicitly instruct the reflector model to avoid including any instance-specific or dialect-specific details in the optimized strategy text, including concrete table names, column names, dialect-specific function syntax, or schema patterns observed in the training batch.
The reflector is told that the skill must generalize across unseen databases, SQL dialects, and task formats.
This constraint is important because the skills optimized on training data with Snowflake grammar are later evaluated on SQLite and BigQuery, making dialect-specific advice actively harmful.

\paragraph{Correctness-only reward signal.}
We deliberately use binary execution correctness as the sole reward signal for GEPA optimization.
Although finer-grained intermediate rewards---such as keyword overlap with gold SQL, structural similarity scores, or partial-credit metrics based on clause matching---might seem more informative, we found them prone to reward hacking in practice.
For example, a keyword-coverage reward incentivizes skills that instruct the agent to speculatively include as many SQL keywords as possible, inflating the reward without improving correctness.
Similarly, structural similarity rewards can penalize valid alternative query plans that differ from the gold SQL in form but not in semantics.
Binary correctness avoids these pathologies: a skill is rewarded only if the agent's final SQL produces the correct result on the target database, providing an unambiguous and ungameable training signal.


\paragraph{Skill-order rotation.}
In Algorithm~\ref{alg:sys}, each skill removes its solved examples before the next skill is optimized, so earlier skills in a batch see an easier and broader residual, while later skills are trained on a narrower and harder subset. 
To guarantee that each skill occasionally sees the full batch before other skills remove solved examples, we rotate the skill order across batches. This reduces positional bias and prevents later skills from being systematically specialized only to the hardest tail of failures. 
Note that if the number of batches is fewer than the number of skills, we apply a larger stride to the rotation to maintain uniform positional coverage.


\newpage
\section{Skill Pool}\label{app:skill-pool}

We list the $K=8$ seed skills used in our experiments, showing both the initial hand-designed seed prompt and the prompt after residual optimization on Snowflake training data.
Each prompt is injected verbatim into the agent's system message; the agent's tool set and control flow remain unchanged.
Optimized prompts are labeled with a round suffix (e.g., \texttt{\_r1}) indicating the GEPA batch that produced the accepted mutation.




\paragraph{default.} \emph{Strategy:} the baseline agentic behavior---balanced exploration followed by incremental SQL construction.

\begin{seedskill}[default]
\#\# Strategy\\
1. EXPLORE first: run queries to understand the data---check table structures, column values, data types, join keys, actual string values in the data\\
2. If unsure about any SQL function's syntax or behavior, call lookup\_docs BEFORE writing the query\\
3. For common patterns (top-N, running totals, pivots), call get\_sql\_pattern for a template\\
4. PLAN your approach based on what you discovered\\
5. WRITE and TEST your SQL incrementally---run it via execute\_sql to check results\\
6. VERIFY results look reasonable (right number of rows, right columns, sensible values)\\
7. Call review\_sql to get a second opinion before submitting\\
8. SUBMIT only when confident
\end{seedskill}

\begin{optskill}[default]
1. EXPLORE the data first (as in seed).\\
\textbf{2. CLARIFY ambiguities---identify potential traps: \textcolor{sqlkw}{NULLs} in key columns, case sensitivity, duplicate rows, date formats, and whether counts should be \textcolor{sqlkw}{DISTINCT}.}\\
\textbf{3. MAP the question to SQL primitives---explicitly decide join type (\textcolor{sqlkw}{INNER} vs \textcolor{sqlkw}{LEFT}), filter placement (\textcolor{sqlkw}{WHERE} vs \textcolor{sqlkw}{HAVING}), aggregation scope, and \textcolor{sqlkw}{NULL} handling before coding.}\\
4. Check templates---call get\_sql\_pattern and lookup\_docs (as in seed).\\
5. BUILD incrementally---write and execute\_sql each \textcolor{sqlkw}{CTE} or subquery alone (as in seed).\\
\textbf{6. VALIDATE against the question---re-read the question, then check: correct columns returned? correct filter conditions? \textcolor{sqlkw}{DISTINCT} where needed? \textcolor{sqlkw}{NULL}-safe denominators? ordering and limits match?}\\
\textbf{7. CROSS-CHECK edge cases---run a quick sanity query (e.g., total counts, min/max values, a spot-check join) to confirm the final result is not inflated by fanout or deflated by over-filtering.}\\
8. REVIEW---call review\_sql and address any flagged issues.\\
9. SUBMIT only after incremental checks and review pass.
\end{optskill}

\paragraph{direct\_coder.} \emph{Strategy:} drafts SQL immediately, refines through execution feedback.

\begin{seedskill}[direct_coder]
\#\# Strategy: DIRECT CODING\\
You are an EFFICIENT SQL writer. Write SQL quickly, test, iterate.\\[4pt]
1. Read the question carefully. Identify the core tables, joins, and aggregations needed.\\
2. Write your best SQL attempt IMMEDIATELY based on the schema.\\
3. Execute it. If errors occur, read the error message carefully and fix.\\
4. If the query runs but results look wrong, investigate specific columns/values.\\
5. Iterate rapidly---each revision should fix one specific issue.\\
6. Do NOT over-explore. Only investigate columns/values that are directly relevant to errors.\\
7. SUBMIT as soon as the query produces reasonable results.
\end{seedskill}

GEPA added lookup-table awareness and structured error-repair guidance. Key additions over the seed (new or substantially expanded material in \textbf{bold}):
\begin{optskill}[direct_coder]
\#\# Strategy: DIRECT CODING\\[4pt]
1. \textbf{Read the schema first.} Before writing any SQL, identify ALL tables mentioned or implied by the question. \textbf{Pay special attention to lookup/reference/static tables (e.g., category tables, node tables, type tables) that provide human-readable names or filter criteria---these almost always require a \textcolor{sqlkw}{JOIN}.}\\
2. \textbf{Map question terms to schema columns.} If the question references a name, label, or category, find which table owns that column. \textbf{Never filter or select on a column that doesn't exist in the target table---use the correct table via \textcolor{sqlkw}{JOIN} instead.}\\
3. Write your best SQL IMMEDIATELY based on the schema. \textbf{Use explicit \textcolor{sqlkw}{JOIN} conditions. When lookup tables exist, join them rather than filtering on raw IDs or payload strings.}\\
4. \textbf{Execute. Fix errors specifically:}\\
\quad - Column-not-found $\to$ find the correct table and \textcolor{sqlkw}{JOIN} it\\
\quad - Wrong results $\to$ verify \textcolor{sqlkw}{JOIN} keys and \textcolor{sqlkw}{WHERE} filters match actual values\\
\quad - Missing rows $\to$ check \textcolor{sqlkw}{JOIN} type (\textcolor{sqlkw}{INNER} vs \textcolor{sqlkw}{LEFT})\\
5. \textbf{Aggregation and output checks:} Verify \textcolor{sqlkw}{GROUP BY} includes all non-aggregated SELECT columns; verify \textcolor{sqlkw}{ORDER BY}, LIMIT, and NULLS LAST where appropriate; confirm output column names match the question.\\
6. Iterate rapidly---each revision fixes one specific issue.\\
7. SUBMIT as soon as the query produces reasonable results.\\[4pt]
\textbf{Key reminder: Questions involving categories, types, nodes, or classifications almost always require joining to a static/lookup table. Never assume the needed label lives in the fact table---check the schema.}
\end{optskill}

\paragraph{decompose.} \emph{Strategy:} identifies sub-questions before composing the final SQL.

\begin{seedskill}[decompose]
\#\# Strategy: DECOMPOSE \& CONQUER\\
Break complex questions into simple subqueries, build bottom-up.\\[4pt]
1. PARSE the question into atomic requirements: What is being counted/summed/averaged? What are the filter conditions? What are the grouping columns? Is there ranking, ordering, or limiting?\\
2. BUILD each piece as a standalone \textcolor{sqlkw}{CTE}: Start with the base data, add joins one at a time verifying row counts, add aggregations verifying results at each step.\\
3. COMPOSE \textcolor{sqlkw}{CTEs} into the final query using WITH...SELECT.\\
4. Run each \textcolor{sqlkw}{CTE} individually via execute\_sql to verify intermediate results.\\
5. Call review\_sql on the assembled final query before submitting.\\
6. SUBMIT only when confident.
\end{seedskill}

GEPA expanded this into a structured six-step process with explicit grain-anchoring and filter-validation phases. Key additions (in \textbf{bold}):
\begin{optskill}[decompose]
\#\# Strategy: DECOMPOSE \& CONQUER\\[4pt]
Step 1: PARSE the question into atomic requirements---\textbf{explicitly identify: output columns/metrics, grain (one row per what?), filter conditions, grouping dimensions, and whether there is ranking, ordering, limiting, or a ratio/composition calculation.}\\
Step 2: \textbf{ANCHOR the grain before grouping.} Match \textcolor{sqlkw}{GROUP BY} columns precisely to the output grain---no more, no less. \textbf{If the question asks for monthly totals, group by month only; do not add route, city, or other columns unless explicitly requested.}\\
Step 3: BUILD each piece as a standalone \textcolor{sqlkw}{CTE}. \textbf{For ratio/composition queries, compute totals in one \textcolor{sqlkw}{CTE}, subgroup counts in another, then \textcolor{sqlkw}{JOIN} and divide.}\\
Step 4: \textbf{VERIFY join logic and filter semantics.} Confirm join keys actually link the intended entities; avoid fan-out. \textbf{Confirm filter values match the domain as they appear in the data, not as paraphrased in the question.}\\
Step 5: ASSEMBLE and REVIEW. Call review\_sql; confirm output columns and grain match the question.\\
Step 6: SUBMIT only when confident.
\end{optskill}

\paragraph{explore\_heavy.} \emph{Strategy:} spends extra steps on schema and sample-row inspection before drafting.

\begin{seedskill}[explore_heavy]
\#\# Strategy: DEEP EXPLORATION\\
You are a THOROUGH explorer. Before writing ANY SQL, deeply understand the data.\\[4pt]
1. CATALOG SCAN: List all schemas and tables. Check which tables have data (SELECT COUNT(*)).\\
2. COLUMN AUDIT: For each relevant table, run DESCRIBE or SHOW COLUMNS. Check actual column types.\\
3. VALUE PROFILING: For key columns, run SELECT \textcolor{sqlkw}{DISTINCT} to see actual values, formats, ranges.\\
4. \textcolor{sqlkw}{JOIN} DISCOVERY: Test joins between tables with small queries before using them in final SQL.\\
5. Only after thorough exploration, write your query.\\
6. Test edge cases: What if there are \textcolor{sqlkw}{NULLs}? What if the join produces duplicates?\\
7. Call review\_sql before submitting.\\
8. SUBMIT only when confident.
\end{seedskill}

GEPA reorganized the strategy into six explicit phases and added an output-requirements clarification phase (Phase~4) that was absent from the seed. Key additions (in \textbf{bold}):
\begin{optskill}[explore_heavy]
\#\# Strategy: DEEP EXPLORATION\\[4pt]
Phase 1--3: Schema Discovery, Value Profiling, Join Validation (as in seed).\\
\textbf{Phase 4: Clarify Output Requirements Before Writing SQL.}\\
\textbf{Before coding, explicitly answer:}\\
\quad\textbf{- Granularity: one summary row, or one row per dimension?}\\
\quad\textbf{- Aggregation: which columns require AVG, SUM, COUNT?}\\
\quad\textbf{- Filters: date ranges, status filters, or categorical constraints implied?}\\
\quad\textbf{- Extra columns: would adding dimension columns change the granularity?}\\
\textbf{If the question asks for summary metrics without specifying a grouping dimension, produce a single aggregated row---do NOT return raw row-level data or add unrequested dimension columns.}\\
Phase 5: Write and Validate SQL---\textbf{apply NULLIF in denominators to avoid divide-by-zero; use \textcolor{sqlkw}{CTEs} to separate filtering, joining, and aggregation stages.}\\
Phase 6: Submit only when the output granularity, aggregation, and columns exactly match what was requested.
\end{optskill}

\paragraph{conservative.} \emph{Strategy:} prefers the simplest faithful query; avoids speculative constructs.

\begin{seedskill}[conservative]
\#\# Strategy: CONSERVATIVE \& SAFE\\
Prefer simple, safe SQL. Avoid unnecessary complexity.\\[4pt]
1. Start with the SIMPLEST possible query that could answer the question.\\
2. Avoid: Complex nested subqueries, window functions unless required, multiple joins when one will do, \textcolor{sqlkw}{HAVING} when \textcolor{sqlkw}{WHERE} suffices, correlated subqueries.\\
3. Do NOT add any clause the question didn't ask for: No ROUND unless asked, no COALESCE unless \textcolor{sqlkw}{NULLs} are a proven problem, no extra \textcolor{sqlkw}{WHERE} filters or output columns.\\
4. When ambiguous, pick the LITERAL interpretation.\\
5. Test your query. If it works and looks reasonable, submit.\\
6. SUBMIT once results make sense. Don't over-iterate.
\end{seedskill}

GEPA added explicit guidance on AND-vs-OR filter logic and output granularity matching---two recurring failure modes in the training data. Key additions (in \textbf{bold}):
\begin{optskill}[conservative]
(Steps 1--3 retained from seed.)\\
\textbf{4. Filter logic---AND vs OR:}\\
\quad\textbf{- AND: both conditions must hold simultaneously on the same row.}\\
\quad\textbf{- OR: either condition suffices.}\\
\quad\textbf{- When filtering across related entities (e.g., origin OR destination), default to OR unless the question explicitly requires all conditions on the same record.}\\
\textbf{5. Output structure---match the question's requested granularity exactly:}\\
\quad\textbf{- Identify the grouping dimensions the question asks for before writing \textcolor{sqlkw}{GROUP BY}.}\\
\quad\textbf{- Do not substitute finer-grained groupings (e.g., day) when coarser ones (e.g., month) are requested.}\\
\quad\textbf{- Column names and aliases should reflect the question's terminology.}\\
6. When ambiguous, pick the LITERAL interpretation. \textbf{Consult reference templates for grouping, join, and conditional aggregation patterns.}\\
7. Test and submit once results make sense.
\end{optskill}

\paragraph{adversarial\_checker.} \emph{Strategy:} actively stress-tests the query for edge cases before finalizing.

\begin{seedskill}[adversarial_checker]
\#\# Strategy: ADVERSARIAL SELF-CHECK\\
After writing SQL, actively try to BREAK it before submitting.\\[4pt]
1. Explore and write your initial SQL query.\\
2. Execute it and get results.\\
3. NOW, CHALLENGE your own query: Does the question ask for X but you computed Y? Could a join be producing duplicates? Are you filtering correctly? For ratios: verify numerator and denominator independently. For `top N': verify ordering.\\
4. Fix any issues you discover.\\
5. Call review\_sql for an independent check.\\
6. SUBMIT only after surviving both your own and the reviewer's scrutiny.
\end{seedskill}

GEPA restructured the strategy into four explicit phases, adding an upfront decomposition step and specific guidance for period-over-period comparisons and join-type selection:
\begin{optskill}[adversarial_checker]
\textbf{Phase 1---Decompose the Question Before Writing:}\\
\quad\textbf{Identify ALL computations required. If the question involves change/comparison, plan for two separate aggregations and a join or pivot---never a single flat filter. If it involves rates or ratios, plan numerator and denominator explicitly. Sketch the output shape.}\\
\textbf{Phase 2---Write SQL Matching Full Complexity:}\\
\quad\textbf{Use \textcolor{sqlkw}{CTEs} for multi-step logic. Use window functions for rankings. Use FULL \textcolor{sqlkw}{OUTER \textcolor{sqlkw}{JOIN}} when comparing two periods where either side may have no data. Avoid the trap of ``just filter and return raw rows'' when computation is required.}\\
Phase 3---Adversarial Challenge (expanded from seed): complexity check, output check, aggregation check, filter check, \textbf{join type check (OUTER vs \textcolor{sqlkw}{INNER})}, ratio/rate check.\\
Phase 4---Fix and Validate. \textbf{Consult templates for period-over-period \textcolor{sqlkw}{CTE} patterns.}
\end{optskill}

\paragraph{template\_first.} \emph{Strategy:} anchors query shape from retrieved SQL patterns before coding.

\begin{seedskill}[template_first]
\#\# Strategy: TEMPLATE-FIRST PLANNER\\
Before writing SQL, call get\_sql\_templates and get\_sql\_pattern to anchor the query shape. Pick the closest template, adapt only the table names, columns, filters, grain, and ordering confirmed from the schema, then execute one focused validation query. Prefer template-guided \textcolor{sqlkw}{CTE}/window/ratio patterns over ad hoc exploration. Submit after the result shape matches the question.
\end{seedskill}

GEPA appended a mandatory pre-submit checklist targeting schema hallucination---a failure mode where the agent references tables or columns it has not verified. Key addition (in \textbf{bold}):
\begin{optskill}[template_first]
(Seed text retained in full.)\\[4pt]
\textbf{MANDATORY pre-submit checklist (do NOT skip):}\\
\textbf{1. Tables exist as named---every table in your final SQL was confirmed to exist via execute\_sql.}\\
\textbf{2. Columns exist as named---every column appeared in actual schema output.}\\
\textbf{3. Identifier quoting is correct for the dialect.}\\
\textbf{4. Final query has executed at least once and returned non-empty rows whose shape matches the question.}\\
\textbf{Submitting SQL that references non-existent columns, or that has never been successfully executed, is an automatic failure.}
\end{optskill}

\paragraph{fast\_error\_repair.} \emph{Strategy:} prioritizes execution-error recovery over upfront planning.

\begin{seedskill}[fast_error_repair]
\#\# Strategy: FAST ERROR REPAIR\\
Move quickly: inspect only the most relevant tables, write the first plausible SQL early, execute it, and repair from concrete errors or wrong-shaped results. Do not exhaustively profile. Each iteration changes one thing: missing column, join key, filter value, aggregation grain, or dialect syntax. Submit as soon as the SQL runs and the output shape answers the question.
\end{seedskill}

GEPA expanded the single-paragraph seed into a five-phase process, adding a structured triage table and a value-sanity-check phase. Key additions (in \textbf{bold}):
\begin{optskill}[fast_error_repair]
\textbf{Core loop: write $\to$ run $\to$ verify values $\to$ repair $\to$ submit.}\\[4pt]
Phase 1---Minimal Table Scan (as in seed).\\
Phase 2---First SQL Draft. \textbf{Prefer \textcolor{sqlkw}{CTEs} for multi-step logic. Commit one join key and one filter assumption up front.}\\
Phase 3---Execute \& Triage. \textbf{Change exactly one thing per iteration:}\\
\quad\textbf{Column not found $\to$ alias or rename}\\
\quad\textbf{Join returns 0 rows $\to$ relax join key or flip direction}\\
\quad\textbf{Count inflated $\to$ add \textcolor{sqlkw}{DISTINCT} or check fan-out join}\\
\quad\textbf{Wrong grain $\to$ re-examine \textcolor{sqlkw}{GROUP BY} columns}\\
\textbf{Phase 4---Value Sanity Check (the key addition):}\\
\quad\textbf{Before submitting: do not only check shape---check values. Are counts suspiciously 0 or implausibly large? Does a percentage exceed 1.0 or go negative?}\\
Phase 5---Submit when both shape and values are plausible.\\[4pt]
\end{optskill}

\newpage
\section{Additional Experimental Results}\label{app:additional-exp}

\subsection{Variance and Candidate-Level Breakdowns}
\Cref{tab:exp:spider2lite-opus-with-std} reports the full Spider2-Lite Opus-4.6 table with standard deviations. The qualitative conclusion from the main text is unchanged after adding variance: \sys remains the strongest method on selected accuracy in all three dialects, with the largest gains on the more complex Snowflake and BigQuery settings.

\begin{table*}[h]
\centering
\scriptsize
\setlength{\tabcolsep}{3pt}
\begin{tabular}{lccccccccc}
\toprule
& \multicolumn{3}{c}{SQLite} & \multicolumn{3}{c}{Snowflake} & \multicolumn{3}{c}{BigQuery} \\
\cmidrule(lr){2-4}\cmidrule(lr){5-7}\cmidrule(lr){8-10}
Method & pass@1 & pass@8 & Sel. acc. & pass@1 & pass@8 & Sel. acc. & pass@1 & pass@8 & Sel. acc. \\
\midrule
\dinsql & 40.74$\pm$2.72 & / & / & 0.97$\pm$0.40 &  &  & 18.54$\pm$1.71 &  &  \\
\reforce & 55.28$\pm$2.62 & / & 57.78$\pm$2.72 & 43.47$\pm$2.09 & / & 50.72$\pm$2.25 & 51.10$\pm$2.23 & / & 55.12$\pm$2.28 \\
\chasesql & 62.59$\pm$2.47 & 76.30$\pm$2.59 & 63.70$\pm$2.59 & 51.21$\pm$1.94 & 68.60$\pm$2.01 & 53.14$\pm$2.33 & 59.02$\pm$2.00 & 71.71$\pm$2.03 & 56.59$\pm$2.28 \\
\sys & 62.13$\pm$2.04 & 73.33$\pm$1.82 & 64.44$\pm$1.62 & 60.08$\pm$1.59 & 72.46$\pm$1.55 & 64.25$\pm$1.52 & 62.07$\pm$1.84 & 73.17$\pm$1.67 & 64.88$\pm$1.12 \\
\bottomrule
\end{tabular}
\caption{Spider2-Lite results with Opus 4.6.}
\label{tab:exp:spider2lite-opus-with-std}
\end{table*}

\paragraph{Why \reforce selection can underperform mean candidate quality.}
In the main experiment, we noted that ReFoRCE's final selected SQL can be worse than the average quality of its individual candidate runs. \Cref{tab:exp:gpt54-reforce-cands} makes this concrete. 
On SQLite, the mean per-candidate Pass@1 is 64.44 while the final selected accuracy is only 58.52; on Snowflake, the same gap is 47.83 vs.\ 41.06. 
This indicates that \reforce often produces correlated candidates that agree on the same wrong answer, so majority voting can amplify a dominant failure mode rather than recover the strongest candidate. BigQuery is less pathological: selected accuracy (51.22) slightly exceeds the mean candidate accuracy (48.19), but still remains well below the oracle upper bound of 56.59. 
Here, ``Max'' is exactly Pass@4 because \reforce exposes four candidate runs in this evaluation.
ReFoRCE's \texttt{-{}-num\_votes} parameter only sets the number of self-refinement threads per instance, not the number of surviving candidates: each thread is a 5-step refine loop that may terminate without writing any SQL (max-iter exhaustion, empty-result early-stop, invalid response, or oversized schema). In our GPT-5.4 run the realized candidate count per instance ranges over all threads, so a Pass@k oracle is ill-defined and not directly comparable to the Pass@8 results we report for the other methods. Pass@1 is defined as a complete set of first generated candidates for each question to show the interior intermediate results of ReFoRCE before majority voting. In this table, we simply give four sets of candidates to explain why ReFoRCE gets lower accuracy on Sel. acc. compared to so-called Pass@1.

\begin{table*}[h]
\centering
\small
\setlength{\tabcolsep}{3pt}
\begin{tabular}{lccccccccc}
\toprule
Dialect & cand 0 & cand 1 & cand 2 & cand 3 & Min & Max & Mean & Sel.\ acc. \\
\midrule
SQLite (n=135)    & 63.70 & 61.48 & 65.93 & 65.93 & 53.33 & 74.81 & 64.44 & 58.52 \\
Snowflake (n=207) & 41.55 & 45.89 & 43.96 & 44.44 & 38.65 & 56.52 & 47.83 & 41.06 \\
BigQuery (n=205)  & 48.29 & 46.83 & 45.37 & 46.34 & 41.95 & 56.59 & 48.19 & 51.22 \\
\bottomrule
\end{tabular}
\caption{\reforce per-candidate Pass@1 on Spider2-Lite for GPT-5.4}
\label{tab:exp:gpt54-reforce-cands}
\end{table*}



\newpage
\subsection{Additional \chasesql Comparison with GPT-5.4}
\paragraph{Candidate-pool quality and rankability.}
The main Spider2-Lite results suggest that \sys's advantage over \chasesql is not only that it contains slightly more correct candidates, but also that its candidate pool is easier to rank. \Cref{tab:chasesql-gpt54-candidate-density,tab:chasesql-gpt54-head-to-head,tab:chasesql-gpt54-paired-significance} provide supporting evidence on the GPT-5.4 Spider2-Lite evaluation. Compared with \chasesql, \sys improves pass@1 (56.95 vs.\ 54.96), pass@8 (78.79 vs.\ 76.42), and selected accuracy (64.90 vs.\ 60.69), while reducing the oracle--selector gap (13.89 vs.\ 15.72). 

The same-instance head-to-head analysis shows 47 \sys-only wins versus 24 \chasesql-only wins on final selection, and McNemar's test confirms that the selected-accuracy gain is statistically significant ($p=0.0086$). The candidate-density view further shows that \sys yields fewer dead pools with zero correct candidates, slightly more rich pools with 6--8 correct candidates, and fewer exact duplicate slots. Together, these results support the claim that residual skill optimization improves both candidate quality and candidate rankability, rather than merely increasing surface-form variation.

\begin{table}[h]
\centering
\small
\setlength{\tabcolsep}{4pt}
\begin{tabular}{lccccc}
\toprule
Method & 0 corr. & 1 corr. & 2--3 corr. & 4--5 corr. & 6--8 corr. \\
\midrule
\chasesql & 23.58 & 6.76 & 10.97 & 11.33 & 47.35 \\
\sys & \textbf{21.21} & 6.95 & 10.60 & 12.43 & \textbf{48.81} \\
\bottomrule
\end{tabular}
\caption{Distribution of the number of correct candidates per instance in the 8-candidate pool. \sys pool yields fewer dead pools (0 correct), slightly more rich pools (6--8 correct).}
\label{tab:chasesql-gpt54-candidate-density}
\end{table}

\begin{table}[h]
\centering
\footnotesize
\setlength{\tabcolsep}{5pt}
\begin{tabular}{lcccc}
\toprule
Dialect & Ours-only sel. & CHASE-only sel. & Ours-only oracle & CHASE-only oracle \\
\midrule
Local & 12 & 5 & 6 & 5 \\
Snowflake & 16 & 8 & 11 & 6 \\
BigQuery & 19 & 11 & 12 & 5 \\
\midrule
Total & 47 & 24 & 29 & 16 \\
\bottomrule
\end{tabular}
\caption{Same-instance head-to-head on Spider2-Lite GPT-5.4. ``Ours-only sel.'' counts instances where the final selected SQL is correct for \sys but not for \chasesql. ``Ours-only oracle'' counts instances where the 8-candidate pool contains at least one correct SQL for \sys but not for \chasesql.}
\label{tab:chasesql-gpt54-head-to-head}
\end{table}

\begin{table}[h]
\centering
\scriptsize
\setlength{\tabcolsep}{5pt}
\begin{tabular}{lcccccc}
\toprule
Comparison & $N$ & \sys & \chasesql & Ours-only wins & CHASE-only wins & Exact $p$ \\
\midrule
Final selected accuracy (all instances) & 547 & \textbf{64.90} & 60.69 & 47 & 24 & \textbf{0.0086} \\
Oracle pass@8 coverage (all instances) & 547 & \textbf{78.79} & 76.42 & 29 & 16 & 0.0725 \\
Final accuracy on jointly solvable instances & 402 & \textbf{86.32} & 81.34 & 39 & 19 & \textbf{0.0119} \\
\bottomrule
\end{tabular}
\caption{Paired same-instance comparisons on Spider2-Lite GPT-5.4 using the exact McNemar test. The strongest effect is on final selected accuracy, including the subset where both methods already contain at least one correct candidate somewhere in the pool. This supports the claim that the low-temperature \sys pool is easier to rank, not merely that it has slightly better oracle coverage.}
\label{tab:chasesql-gpt54-paired-significance}
\end{table}

\clearpage


\end{document}